\documentclass[10pt,journal,compsoc]{IEEEtran}
\ifCLASSOPTIONcompsoc
  \usepackage[nocompress]{cite}
\else
  \usepackage{cite}
\fi

\usepackage[pdftex]{graphicx}
\usepackage{tikz,pgfplots}
\pgfplotsset{compat=1.14} 
\tikzstyle{vertex} = [fill,shape=circle,node distance=30pt]
\tikzstyle{edge} = [fill,opacity=.6,fill opacity=.5,line cap=round, line join=round, line width=10pt]
\tikzstyle{elabel} =  [fill,shape=circle,node distance=30pt,fill opacity=.9]

\pgfdeclarelayer{background}
\pgfsetlayers{background,main}

\ifCLASSINFOpdf
\else
\fi
\usepackage{amsmath}
\usepackage{amssymb}
\usepackage[mathscr]{euscript}
\usepackage{mathtools}
\usepackage[utf8]{inputenc}
\usepackage[english]{babel}
\usepackage{amsthm}
\usepackage{algorithmic}
\usepackage[ruled,vlined]{algorithm2e}

\usepackage{array}

\newtheorem{theorem}{Theorem}
\newtheorem{definition}{Definition}
\newtheorem{proposition}{Proposition}
\newtheorem{corollary}{Corollary}
\newtheorem{lemma}{Lemma}

\usepackage[most]{tcolorbox}

\usepackage{url}

\begin{document}
%
\title{Tensor Entropy for Uniform Hypergraphs}
%
%
%
%

\author{Can~Chen, ~Indika~Rajapakse
\IEEEcompsocitemizethanks{\IEEEcompsocthanksitem C. Chen is with the Department of Mathematics and the Department of Electrical Engineering and Computer Science, University of Michigan, Ann Arbor,
MI, 48109.\protect\\
E-mail: canc@umich.edu
\IEEEcompsocthanksitem I. Rajapakse is with the Department of Computational Medicine \& Bioinformatics, Medical School and the Department of Mathematics, University of Michigan, Ann Arbor, MI, 48109.
\protect\\
E-mail: indikar@umich.edu}
\thanks{Manuscript received February 28, 2020; revised May 13, 2020.}}

%
%

\markboth{IEEE Transactions on Network Science and Engineering,~Vol.~X, No.~X, X~20XX}%
{Chen \MakeLowercase{\textit{et al.}}: Tensor Entropy for Uniform Hypergraphs}
%



\IEEEtitleabstractindextext{%
\begin{abstract}
In this paper, we develop the notion of entropy for uniform hypergraphs via tensor theory. We employ the probability distribution of the generalized singular values, calculated from the higher-order singular value decomposition of the Laplacian tensors, to fit into the Shannon entropy formula. We show that this tensor entropy is an extension of von Neumann entropy for graphs. In addition, we establish results on the lower and upper bounds of the entropy and demonstrate that it is a measure of regularity for uniform hypergraphs in simulated and experimental data. We exploit the tensor train decomposition in computing the proposed tensor entropy efficiently. Finally, we introduce the notion of robustness for uniform hypergraphs.
\end{abstract}

}

\maketitle

\IEEEdisplaynontitleabstractindextext

%
\IEEEpeerreviewmaketitle

\IEEEraisesectionheading{\section{Introduction}\label{sec:introduction}}
\IEEEPARstart{M}{any} real world complex systems can be decomposed and analyzed using networks. There are two classical types of complex networks, scale-free networks and small world networks, which have provided insights in social sciences, cell biology, neuroscience and computer science \cite{Amaral11149,scale_free,small_world}. Recent advancements in genomics technology, such as genome-wide chromosome conformation capture (Hi-C), have inspired us to consider the human genome as a dynamic graph \cite{Rajapakse711,RIED20171}. Studying such dynamic graphs often requires identifying the changes in network properties, such as degree distribution, path lengths and clustering coefficients \cite{dynamic_network,MONNIG2018347,Ranshous:2015:ADD:3160212.3160218}.

Numerous techniques have been developed for anomaly detection based on evaluating similarities between graphs \cite{a19019efc43b4c0cab1fb8c22b8e040e,10.1007/s13174-010-0003-x}. A classic approach for detecting anomalous timestamps during the evolution of dynamic graphs is comparing two consecutive graphs using distance or similarity functions. A comprehensive survey on similarity measures can be found in \cite{cha2007comprehensive}. Two popular measures, Hamming distance \cite{6772729} and Jaccard distance \cite{10.1038/234034a0}, are often problem-specific and sensitive to small perturbations or scaling, thus providing limited understanding of variations between graphs \cite{10.1214/18-AOAS1176}. Model-agnostic approaches, such as eigenvalue based/spectral measures, are more flexible in their representations and interpretations. Therefore, these approaches can more appropriately quantify the global structural complexity of graphs \cite{Simonyi1993GraphEA,10.1145/1134271.1134282,7456290}.

The von Neumann entropy of a graph, first introduced by Braunstein et al. \cite{graph_entropy}, is a spectral measure used in structural pattern recognition. The intuition behind the measure is linking the graph Laplacian to density matrices from quantum mechanics, and measuring the complexity of the graphs via the von Neumman entropy of the corresponding density matrices \cite{10.1093/comnet/cny028}. Additionally, the measure can be viewed as the information theoretic Shannon entropy, i.e., 
\begin{equation}\label{eq:1}
\textsc{S} = -\sum_{j} \eta_j\ln{\eta_j},
\end{equation}
where, $\eta_j$ are the normalized eigenvalues of the Laplacian matrix of a graph such that $\sum_j \eta_j=1$. Passerini and Severini \cite{Passerini2008TheVN} observed that the von Neumman entropy of a graph tends to grow with the number of connected components, the reduction of long paths and the increase of nontrivial symmetricity, and suggested that it can be viewed as a measure of regularity. They also showed that the entropy (\ref{eq:1}) is upper bounded by $\ln{(n-1)}$ where $n$ is the number of vertices of a graph.

However, most data representations are multidimensional, and using graph models to describe them may lose information \cite{7761624}.  A hypergraph is a generalization of a graph in which a hyperedge can join any number of vertices \cite{hypergraph1}, see Figure \ref{fig:0}. Thus, hypergraphs can represent multidimensional relationship unambiguously \cite{7761624}. Examples of hypergraphs include co-authorship networks, film actor/actress networks and protein-protein interaction networks \cite{newman2010networks}. The authors in \cite{7761624} also mention that hypergraphs require less storage space than graphs which may accelerate computation. Moreover, a hypergraph can be represented by a tensor if its hyperedges contain the same number of vertices (referred to as a uniform hypergraph). Tensors are multidimensional arrays generalized from vectors and matrices, preserving multidimensional patterns and capturing higher-order interactions and couplings within multiway data \cite{chen_2020}. Tensor decompositions such as CANDECOMP/PARAFAC decomposition (CPD), higher-order singular value decomposition (HOSVD) and tensor train decomposition (TTD) can facilitate efficient computations and contain contextual interpretations regarding the target data tensors \cite{8187112, 4359192,WANG201986,WILLIAMS20181099}. 

\definecolor{mygray}{gray}{0.95}

\begin{figure}[hbt!]
\centering
\tcbox[colback=mygray]{
\begin{tikzpicture}[scale=1.5]
\node[vertex,text=white,scale=0.7] (v1) {1};
\node[vertex,below of=v1,text=white,scale=0.7, node distance=20pt] (v2) {2};
\node[vertex,below of=v2,text=white,scale=0.7, node distance=20pt] (v3) {3};
\node[vertex,right of=v3,mygray,text=white,scale=0.7, node distance=17.32pt] (v4) {};
\node[vertex,below of=v4,text=white,scale=0.7, node distance=10pt] (v4) {4};
\node[vertex,right of=v4,mygray,text=white,scale=0.7, node distance=17.32pt] (v5) {};
\node[vertex,below of=v5,text=white,scale=0.7, node distance=10pt] (v6) {5};
\node[vertex,left of=v3,mygray,text=white,scale=0.7, node distance=17.32pt] (v7) {};
\node[vertex,below of=v7,text=white,scale=0.7, node distance=10pt] (v8) {6};
\node[vertex,left of=v8,mygray,text=white,scale=0.7, node distance=17.32pt] (v9) {};
\node[vertex,below of=v9,text=white,scale=0.7, node distance=10pt] (v10) {7};
\node[below of=v3,yshift=-3mm]  (A) {\textbf{A}};

\node[vertex,right of=v1,text=white,scale=0.7,node distance=120pt] (v11) {1};
\node[vertex,below of=v11,text=white,scale=0.7, node distance=20pt] (v12) {2};
\node[vertex,below of=v12,text=white,scale=0.7, node distance=20pt] (v13) {3};
\node[vertex,right of=v13,mygray,text=white,scale=0.7, node distance=17.32pt] (v14) {};
\node[vertex,below of=v14,text=white,scale=0.7, node distance=10pt] (v14) {4};
\node[vertex,right of=v14,mygray,text=white,scale=0.7, node distance=17.32pt] (v15) {};
\node[vertex,below of=v15,text=white,scale=0.7, node distance=10pt] (v16) {5};
\node[vertex,left of=v13,mygray,text=white,scale=0.7, node distance=17.32pt] (v17) {};
\node[vertex,below of=v17,text=white,scale=0.7, node distance=10pt] (v18) {6};
\node[vertex,left of=v18,mygray,text=white,scale=0.7, node distance=17.32pt] (v19) {};
\node[vertex,below of=v19,text=white,scale=0.7, node distance=10pt] (v20) {7};
\node[below of=v13,yshift=-3mm]  (B) {\textbf{B}};
\begin{pgfonlayer}{background}
\begin{scope}[transparency group,opacity=.9]
\draw[edge,color=orange] (v1) -- (v2) -- (v3);
\draw[edge,color=red] (v3) -- (v4) -- (v6);
\draw[edge,color=blue] (v3) -- (v8) -- (v10);

\draw[edge,color=orange] (v11) -- (v12);
\draw[edge,color=red] (v12) -- (v13) -- (v14) -- (v16);
\draw[edge,color=blue] (v13) -- (v18) -- (v20);

\end{scope}
\end{pgfonlayer}
\end{tikzpicture}
}
\caption{\textbf{Hypergraphs.} (A) A 3-uniform hypergraph with hyperedges $e_1=\{1,2,3\}$, $e_2=\{3,4,5\}$ and $e_3=\{3,6,7\}$. (B) A non-uniform hypergrah with hyperedges $e_1=\{1,2\}$, $e_2=\{2,3,4,5\}$ and $e_3=\{3,6,7\}$.}
\label{fig:0}
\end{figure}
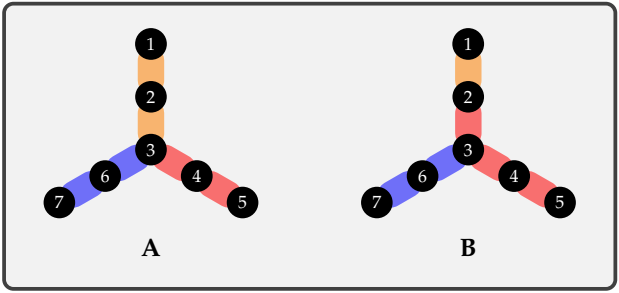

Hypergraph entropy has been recently explored by Hu et al. \cite{Hu2019}, Bloch et al. \cite{10.1007/978-3-030-14085-4_12} and Wang et al. \cite{8283797}.  In \cite{Hu2019}, Hu et al. utilized the probability distribution of the vertex degrees to fit into the Shannon entropy formula and established its lower and upper bounds for special hypergraphs. The degree-based entropy solely depends on the degree distributions of hypergraphs, thus failing to capture comprehensive information, such as path lengths and clustering patterns. Similarly, \cite{10.1007/978-3-030-14085-4_12} defined a hypergraph entropy using incidence matrices, but this formulation may lose higher-order structural information hidden in the hypergraphs, such as nontrivial symmetricity. Furthermore, \cite{8283797} constructed a hypergraph entropy using vertex weighting scores, calculated from a density estimate technique, to select significant lines for model fitting. The score of a vertex (i.e, a fitting line) relies on residuals measured with the Sampson distance under some kernel functions from the line to the data points. Hence, the entropy is difficult to compute and cannot be directly applied to general forms of hypergraphs. 

Based on the works \cite{doi:10.1137/S0895479896305696,Passerini2008TheVN,10.1093/comnet/cny028}, we present a new spectral measure called tensor entropy, which can decipher topological attributes of uniform hypergraphs. For example,  the hypergraph in Figure \ref{fig:0}A with an additional hyperedge $e_4=\{2,4,6\}$ has a higher tensor entropy than that with $e_4=\{1,2,4\}$ because the latter destroys the hypergraph symmetrcity with longer average path length. The key contributions of this paper are as follows:
\begin{itemize}
\item We introduce a new notion of entropy for uniform hypergraphs based on the HOSVD of the corresponding Laplacian tensors. We establish results on the lower and upper bounds of the proposed tensor entropy, and provide a formula for computing the entropy of complete uniform hypergraphs.
\item We adapt a fast and memory efficient TTD-based computational framework in computing the proposed tensor entropy for uniform hypergraphs. 
\item We create two simulated datasets, a hyperedge growth model and a Watts-Strogatz model for uniform hypergraphs. We demonstrate that the proposed tensor entropy is a measure of regularity relying on the vertex degrees, path lengths, clustering coefficients and nontrivial symmetricity for uniform hypergraphs. Further, we present applications to three real world examples:  a primary school contact dataset, a mouse neuron endomicroscopy dataset and a cellular reprogramming dataset. The final example demonstrates the efficacy of the TTD-based computational framework in computing the proposed tensor entropy.
\item We perform preliminary explorations of tensor eigenvalues in the entropy computation and the notion of robustness for uniform hypergraphs. 
\end{itemize}

The paper is organized into five sections. We first introduce the basics of tensor algebra including tensor products, tensor unfolding, higher-order singular value decomposition and tensor train decomposition in section \ref{sec:2.1}. In section \ref{sec:2.2}, we discuss the notion of uniform hypergraphs and extend graph-based definitions to describe uniform hypergraphs' structural properties. We then propose a new form of entropy for uniform hypergraphs with several theoretical results in section \ref{sec:2.4}. In section \ref{sec:2.3}, we exploit the tensor train decomposition to accelerate the tensor entropy computation. Six numerical examples are presented in section \ref{sec:3}.  Finally, we discuss directions for future research in section \ref{sec:dis}  and conclude in section \ref{sec:5}. For ease of reading, we provide a notation table in Appendix A.

\section{Method}
\subsection{Tensor preliminaries}\label{sec:2.1}

We take most of the concepts and notations for tensor algebra from the comprehensive works of Kolda et al. \cite{doi:10.1137/07070111X, Kolda06multilinearoperators}. A \textit{tensor} is a multidimensional array. The \textit{order} of a tensor is the number of its dimensions, also known as \textit{modes}. A $k$-th order tensor usually is denoted by $\textsf{X}\in \mathbb{R}^{n_1\times n_2\times  \dots \times n_k}$.  It is therefore reasonable to consider scalars $x\in\mathbb{R}$ as zero-order tensors, vectors $\textbf{v}\in\mathbb{R}^{n}$ as first-order tensors, and matrices $\textbf{A}\in\mathbb{R}^{m\times n}$ as second-order tensors. For a third-order tensor, \textit{fibers} are named as \textit{column} ($\textsf{X}_{:j_2j_3}$), \textit{row} ($\textsf{X}_{j_1:j_3}$) and \textit{tube} ($\textsf{X}_{j_1j_2:}$), while \textit{slices} are named as \textit{horizontal} ($\textsf{X}_{j_1::}$), \textit{lateral} ($\textsf{X}_{:j_2:}$) and \textit{frontal} ($\textsf{X}_{::j_3}$), see Figure \ref{fig1}. A tensor is called \textit{cubical} if every mode is the same size, i.e., $\textsf{X}\in \mathbb{R}^{n\times n\times  \dots \times n}$. A cubical tensor \textsf{X} is called \textit{supersymmetric} if $\textsf{X}_{j_1j_2\dots j_k}$ is invariant under any permutation of the indices, and is called \textit{diagonal} if $\textsf{X}_{j_1j_2\dots j_k}= 0$ except $j_1=j_2=\dots=j_k$.

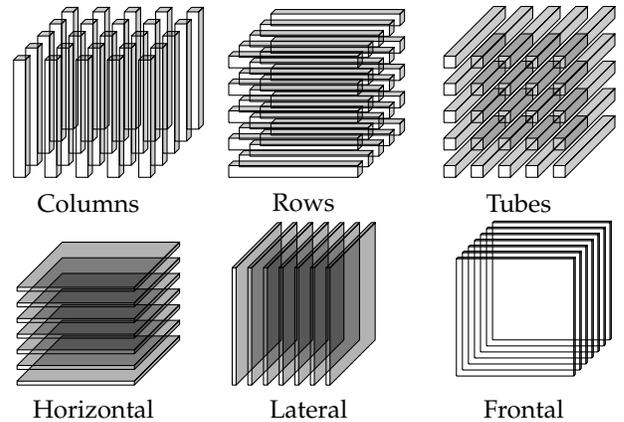
\begin{figure}[ht]
\centering
\begin{tikzpicture}[scale=0.52]
\foreach \x in {0,0.8,1.6,2.4,3.2} {
	\foreach \z in {0,0.8,1.6,2.4,3.2} {
\path (\x,0,0.3-\z) coordinate (A)     (\x+0.3,0,0.3-\z) coordinate (B)
           (\x+0.3,0,-\z) coordinate (C)    (\x,0,-\z) coordinate (D)
           (\x,3,0.3-\z) coordinate (E)     (\x+0.3,3,0.3-\z) coordinate (F)
           (\x+0.3,3,-\z) coordinate (G)      (\x,3,-\z) coordinate (H);
\draw [fill=black,rounded corners = 0.1pt,fill opacity=0.2] (A) -- (B) -- (C) -- (G) -- (F) -- (B)  (A)--(E)--(F)--(G)--(H) --(E);
}
}
\draw[xshift=1.8cm,yshift=-0.3cm]
node[below]
{Columns};

\foreach \y in {0,0.7,1.4,2.1,2.8} {
	\foreach \z in  {0,0.7,1.4,2.1,2.8} {
\path (0+5.5,\y,0.3-\z) coordinate (A)     (3+5.8,\y,0.3-\z) coordinate (B)
           (3+5.8,\y,-\z) coordinate (C)    (0+5.5,\y,-\z) coordinate (D)
           (0+5.5,\y+0.3,0.3-\z) coordinate (E)     (3+5.8,\y+0.3,0.3-\z) coordinate (F)
           (3+5.8,\y+0.3,-\z) coordinate (G)      (0+5.5,\y+0.3,-\z) coordinate (H);
\draw [fill=black,rounded corners = 0.1pt,fill opacity=0.2] (A) -- (B) -- (C) -- (G) -- (F) -- (B)  (A)--(E)--(F)--(G)--(H) --(E);
}
}
\draw[xshift=7.3cm,yshift=-0.3cm]
node[below]
{Rows };

\foreach \x in {0,0.7,1.4,2.1,2.8}{
	\foreach \y in {0,0.7,1.4,2.1,2.8} {
\path (\x+11,\y,0.3-0) coordinate (A)     (\x+11.3,\y,0.3-0) coordinate (B)
           (\x+11.3,\y,-3) coordinate (C)    (\x+11,\y,-3) coordinate (D)
           (\x+11,\y+0.3,0.3-0) coordinate (E)     (\x+11.3,\y+0.3,0.3-0) coordinate (F)
           (\x+11.3,\y+0.3,-3) coordinate (G)      (\x+11,\y+0.3,-3) coordinate (H);
\draw [fill=black,rounded corners = 0.1pt,fill opacity=0.2] (A) -- (B) -- (C) -- (G) -- (F) -- (B)  (A)--(E)--(F)--(G)--(H) --(E);
}
}
\draw[xshift=12.8cm,yshift=-0.3cm]
node[below]
{Tubes };
\end{tikzpicture}
\hspace{0.2cm}

\begin{tikzpicture}[scale=0.52]
\foreach \y in  {0,0.4,0.8,1.2,1.6,2,2.4} {
\path (0,\y,3) coordinate (A)     (3,\y,3) coordinate (B)
           (3,\y,0) coordinate (C)    (0,\y,0) coordinate (D)
           (0,\y+0.1,3) coordinate (E)     (3,\y+0.1,3) coordinate (F)
           (3,\y+0.1,0) coordinate (G)      (0,\y+0.1,0) coordinate (H);
\draw [fill=black,rounded corners = 0.1pt,fill opacity=0.3] (A) -- (B) -- (C) -- (G) -- (F) -- (B)  (A)--(E)--(F)--(G)--(H) --(E);
}
\draw[xshift=0.8cm,yshift=-1.3cm]
node[below]
{Horizontal };

\foreach \x in {0,0.4,0.8,1.2,1.6,2,2.4} {
\path (\x+5.5,0,3) coordinate (A)     (\x+0.1+5.5,0,3) coordinate (B)
           (\x+0.1+5.5,0,0) coordinate (C)    (\x+5.5,0,0) coordinate (D)
           (\x+5.5,3,3) coordinate (E)     (\x+0.1+5.5,3,3) coordinate (F)
           (\x+0.1+5.5,3,0) coordinate (G)      (\x+5.5,3,0) coordinate (H);
\draw [fill=black,rounded corners = 0.1pt,fill opacity=0.3] (A) -- (B) -- (C) -- (G) -- (F) -- (B)  (A)--(E)--(F)--(G)--(H) --(E);
}
\draw[xshift=6.3cm,yshift=-1.3cm]
node[below]
{Lateral };

\foreach \z in {0,0.4,0.8,1.2,1.6,2,2.4} {
\path (11,0,\z) coordinate (A)     (3+11,0,\z) coordinate (B)
           (3+11,0,\z-0.1) coordinate (C)    (0+11,0,\z-0.1) coordinate (D)
           (0+11,3,\z) coordinate (E)     (3+11,3,\z) coordinate (F)
           (3+11,3,\z-0.1) coordinate (G)      (0+11,3,\z-0.1) coordinate (H);
\draw [fill=black,rounded corners = 0.1pt,fill opacity=0.5] (A) -- (B) -- (C) -- (G) -- (F) -- (B)  (A)--(E)--(F)--(G)--(H) --(E);
}
\draw[xshift=11.8cm,yshift=-1.3cm]
node[below]
{Frontal};

\end{tikzpicture}
\hspace{0.2cm}
\caption{\textbf{Fibers and slices of a third-order tensor.} The figure is adapted from \cite{doi:10.1137/07070111X}.}
\label{fig1}
\end{figure}

There are several notions of tensor products. The \textit{inner product} of two tensors $\textsf{X},\textsf{Y}\in \mathbb{R}^{n_1\times n_2\times \dots \times n_k}$ is defined as
\begin{equation}
\langle \textsf{X},\textsf{Y}\rangle =\sum_{j_1=1}^{n_1}\dots \sum_{j_k=1}^{n_k}\textsf{X}_{j_1j_2\dots j_k}\textsf{Y}_{j_1j_2\dots j_k},
\end{equation}
leading to the \textit{tensor Frobenius norm} $\|\textsf{X}\|^2=\langle \textsf{X},\textsf{X}\rangle$. We say two tensors \textsf{X} and \textsf{Y} are \textit{orthogonal} if the inner product $\langle \textsf{X},\textsf{Y}\rangle=0$. The \textit{matrix tensor multiplication} $\textsf{X} \times_{p} \textbf{A}$ along mode $p$ for a matrix $\textbf{A}\in  \mathbb{R}^{m\times n_p}$ is defined by
\begin{equation}
(\textsf{X} \times_{p} \textbf{A})_{j_1j_2\dots j_{p-1}ij_{p+1}\dots j_k}=\sum_{j_p=1}^{n_p}\textsf{X}_{j_1j_2\dots j_p\dots j_k}\textbf{A}_{ij_p}.
\end{equation}
This product can be generalized to what is known as the \textit{Tucker product}, for $\textbf{A}_p\in \mathbb{R}^{m_p\times n_p}$,
\begin{equation}\label{eq5}
\begin{split}
&\textsf{X}\times_1 \textbf{A}_1 \times_2\textbf{A}_2\times_3\dots \times_{k}\textbf{A}_k\\=&\textsf{X}\times\{\textbf{A}_1,\textbf{A}_2,\dots,\textbf{A}_k\}\in  \mathbb{R}^{m_1\times m_2\times\dots \times m_k}.
\end{split}
\end{equation}

\textit{Tensor unfolding} is considered as a critical operation in tensor computations \cite{doi:10.1137/110820609,chen_2019}. The \textit{$p$-mode unfolding} of a tensor $\textsf{X}\in \mathbb{R}^{n_1\times n_2\times  \dots \times n_k}$, denoted by $\textbf{X}_{(p)}$, is defined by 
\begin{equation}
\textsf{X}_{j_1j_2\dots j_k} = (\textbf{X}_{(p)})_{j_pj} \text{ for } j=1+\sum_{\substack{m=1 \\ m\neq p}}^k(j_m-1)\prod_{\substack{i=1 \\ i\neq p}}^{m-1}n_i.
\end{equation}
The ranks of the $p$-mode unfoldings are called \textit{multilinear ranks} of \textsf{X}, which are related to the so-called Higher-Order Singular Value Decomposition (HOSVD), a multilinear generalization of the matrix Singular Value Decomposition (SVD) \cite{5447070, doi:10.1137/S0895479896305696}.

\begin{theorem}[HOSVD] \label{thm:2.1}
A tensor $\textsf{X}\in\mathbb{R}^{n_1\times n_2\times \dots \times n_k}$ can be written as
\begin{equation}
\textsf{X} = \textsf{S}\times_1 \textbf{U}_1\times_2\dots \times_k \textbf{U}_k,
\end{equation}
where, $\textbf{U}_p\in\mathbb{R}^{n_p\times n_p}$ are orthogonal matrices, and $\textsf{S}\in\mathbb{R}^{n_1\times n_2\times \dots \times n_k}$ is a tensor of which the subtensors $\textsf{S}_{j_p=\alpha}$, obtained by fixing the $p$-th index to $\alpha$, have the properties of
\begin{enumerate}
\item all-orthogonality: two subtensors $\textsf{S}_{j_p=\alpha}$ and $\textsf{S}_{j_p=\beta}$ are orthogonal for all possible values of $p$, $\alpha$ and $\beta$ subject to $\alpha\neq \beta$;
\item ordering: $\|\textsf{S}_{j_p=1}\|\geq \dots \geq \|\textsf{S}_{j_p=n_p}\|\geq 0$ for all possible values of $p$.
\end{enumerate}
The Frobenius norms $\|\textsf{S}_{j_p=j}\|$, denoted by $\gamma_{j}^{(p)}$, are the $p$-mode singular values of $\textsf{X}$.
\end{theorem}
De Lathauwer et al. \cite{doi:10.1137/S0895479896305696} showed that the $p$-mode singular values from the HOSVD of \textsf{X} are the singular values of the $p$-mode unfoldings $\textbf{X}_{(p)}$. In section \ref{sec:2.4}, we will use the notion of $p$-mode singular values as the main tool to define the tensor entropy for uniform hypergraphs. 

The Tensor Train Decomposition (TTD) of an $k$-th order tensor $\textsf{X}\in\mathbb{R}^{n_1\times n_2\times \dots \times n_k}$ is given by
\begin{equation}\label{train}
\textsf{X} = \sum_{r_{k}=1}^{R_{k}}\dots\sum_{r_0=1}^{R_0}\textsf{X}^{(1)}_{r_0:r_1} \circ \textsf{X}^{(2)}_{r_1:r_2}\circ \dots \circ \textsf{X}^{(k)}_{r_{k-1}:r_k},
\end{equation}
where, $\circ$ is the \textit{vector outer product}, $\{R_0,R_1,\dots,R_k\}$ is the set of \textit{TT-ranks} with $R_0=R_k=1$, and $\textsf{X}^{(p)}\in\mathbb{R}^{R_{p-1}\times n_p\times R_{p}}$ are called the \textit{core tensors} of the TTD.  There exist optimal TT-ranks for the TTD such that
\begin{equation*}
R_p = \texttt{rank}(\texttt{reshape}(\textsf{X}, \prod_{j=1}^pn_j,\prod_{j=p+1}^k n_j)),
\end{equation*}
for $p=1,2,\dots,k-1$ \cite{doi:10.1137/090752286}. A core tensor $\textsf{X}^{(p)}$ is called \textit{left-orthonormal} if $(\bar{\textbf{X}}^{(p)})^\top\bar{\textbf{X}}^{(p)} = \textbf{I}\in\mathbb{R}^{R_p\times R_p}$, and is called \textit{right-orthonormal} if $\underline{\textbf{X}}^{(p)}(\underline{\textbf{X}}^{(p)})^\top = \textbf{I}\in\mathbb{R}^{R_{p-1}\times R_{p-1}}$ where
\begin{equation*}
\begin{split}
\bar{\textbf{X}}^{(p)} &= \texttt{reshape}(\textsf{X}^{(p)},R_{p-1}n_p,R_p),\\
\underline{\textbf{X}}^{(p)} &= \texttt{reshape}(\textsf{X}^{(p)},R_{p-1},n_pR_p),
\end{split}
\end{equation*}
are the \textit{left-} and \textit{right-unfoldings} of the core tensor, respectively. Here \textbf{I} denotes the identity matrix, and \texttt{rank} and \texttt{reshape} refer to the rank and reshape operations in MATLAB, respectively. Basic tensor algebra, such as addition, tensor products and norms, can be done in the TT-format without requiring to recover back to the full tensor representation \cite{doi:10.1137/090752286}.

\subsection{Uniform hypergraphs}\label{sec:2.2}
We first present some fundamental concepts of hypergraphs \cite{qi2013,HU20151,XIE20132195}. A \textit{hypergraph} \textsf{G} = \{\textbf{V}, \textbf{E}\} where $\textbf{V} = \{1,2,\dots,n\}$ is the vertex set and $\textbf{E} = \{e_1,e_2,\dots,e_m\}$ is the \textit{hyperedge} set with $e_p\subseteq \textbf{V}$ for $p=1,2,\dots,m$ . Two vertices are called \textit{adjacent} if they are in the same hyperedge. A hypergraph is called \textit{connected} if given two vertices, there is a path connecting them through hyperedges. If all hyperedges contain the same number of nodes, i.e., $|e_p| = k$ ($|\cdot|$ denotes the  cardinality of a set), \textsf{G} is called a \textit{$k$-uniform hypergraph}, see Figure \ref{fig:0}A. Significantly, every $k$-uniform hypergraph can be represented by a tensor. 
\begin{definition}
Let \textsf{G} = \{\textbf{V}, \textbf{E}\} be a $k$-uniform hypergraph with $n$ vertices. The adjacency tensor $\textsf{A}\in\mathbb{R}^{n\times n\times \dots\times n}$, which is a $k$-th order $n$-dimensional supersymmetric tensor, is defined as
\begin{equation}
\textsf{A}_{j_1j_2\dots j_k} = \begin{cases} \frac{1}{(k-1)!} \text{ if $(j_1,j_2,\dots,j_k)\in \textbf{E}$}\\ \\0, \text{ otherwise}\end{cases}.
\end{equation}
\end{definition}
The \textit{degree tensor} \textsf{D} of a hypergraph \textsf{G}, associated with \textsf{A}, is a $k$-th order $n$-dimensional diagonal tensor with $\textsf{D}_{jj\dots j}$ equal to the number of hyperedges that consist of  $v_j$ for $j=1,2,\dots,n$. If $\textsf{D}_{jj\dots j} = d$ for all $j$, then \textsf{G} is called \textit{$d$-regular}. Given any $k$ vertices, if they are contained in one hyperedge, then \textsf{G} is called \textit{complete}. 
\begin{definition}
Let \textsf{G} = \{\textbf{V}, \textbf{E}\} be a $k$-uniform hypergraph with $n$ vertices. The Laplacian tensor $\textsf{L}\in\mathbb{R}^{n\times n\times \dots \times n}$ of \textsf{G}, which is a $k$-th order $n$-dimensional supersymmetric tensor, is defined as
\begin{equation}
\textsf{L} = \textsf{D} - \textsf{A},
\end{equation}
where, \textsf{D} and \textsf{A} are the degree and adjacency tensors of \textsf{G}, respectively.
\end{definition}

The Laplacian tensors of uniform hypergraphs possess many similar properties as Laplacian matrices. For example, the smallest \textit{H-eigenvalue} of \textsf{L} is always zero corresponding to the all-one \textit{H-eigenvector} \cite{qi2013}. Moreover, Chen et al. \cite{doi:10.1137/16M1094828} showed that the \textit{Z-eigenvector} associated with the second smallest \textit{Z-eigenvalue} of a normalized Laplacian tensor can be used for hypergraph partition. Detailed descriptions of tensor eigenvalues can be found in Appendix B. 
In the following, we extend several graph-based definitions to describe the  structural properties of uniform hypergraphs.
\begin{definition}
Given a hypergraph \textsf{G}, the index of dispersion of the vertex degree distribution of \textsf{G} is defined to be the ratio of its variance to its mean. 
\end{definition}
\begin{definition}
Given a $k$-uniform hypergraph \textsf{G} with $n$ vertices, the average path length of \textsf{G} is defined by 
\begin{equation}
L_{\text{avg}} = \frac{1}{n(n-1)}\sum_{j\neq i}d(v_j,v_i),
\end{equation}
where, $d(v_j,v_i)$ denotes the shortest distance between $v_j$ and $v_i$. 
\end{definition}
\begin{definition}
Given a $k$-uniform hypergraph \textsf{G} with $n$ vertices, the average clustering coefficient of \textsf{G} is defined by
\begin{equation}
\begin{split}
C_j &= \frac{|\{e_{ilk}: v_i,v_l,v_k\in \textbf{N}_j, e_{ilk}\in\textbf{E}\}|}{\binom {|\textbf{N}_j|}{k}},\\& \Rightarrow C_{\text{avg}} =\frac{1}{n}\sum_{j=1}^nC_j, 
\end{split}
\end{equation}
where, $\textbf{N}_j$ is the set of vertices that are immediately connected to $v_j$, and $\binom {|\textbf{N}_j|}{k} = \frac{|\textbf{N}_j|!}{(|\textbf{N}_j|-k)!k!}$ returns the binomial coefficients. If $|\textbf{N}_j|<k$, we set $C_j=0$.
\end{definition}
\begin{definition}
Given a $k$-uniform hypergraph \textsf{G}, the small world coefficient of \textsf{G} is defined by 
\begin{equation}
\sigma = \frac{C_{\text{avg}}/C_{\text{rand}}}{L_{\text{avg}}/L_{\text{rand}}},
\end{equation} 
where, $C_{\text{avg}}$ and $L_{\text{avg}}$ are the average clustering coefficient and path length of \textsf{G}, respectively, and $C_{\text{rand}}$ and $L_{\text{rand}}$ are the same quantities of its equivalent random uniform hypergraph.
\end{definition}
The equivalent random uniform hypergraphs of \textsf{G} can be constructed analogously as Erd\H{o}s-R\'enyi graphs \cite{10.1371/journal.pone.0002051}, i.e., randomly generating uniform hypergraphs that share the same numbers of vertices and hyperedges with \textsf{G}. All these definitions can be used for quantifying the performance of entropy measures for uniform hypergraphs.

\subsection{Tensor entropy}\label{sec:2.4}
Similar to von Neumann entropy, we exploit the spectrum of Laplacian tensors to define the notion of \textit{tensor entropy} for uniform hypergraphs. 

\begin{definition}
Let \textsf{G} be a $k$-uniform hypergraph with $n$ vertices. The tensor entropy of \textsf{G} is defined by
\begin{equation}\label{eq:3}
\textsc{S} = -\sum_{j=1}^{n} \hat{\gamma}_j\ln{\hat{\gamma}_j},
\end{equation}
where, $\hat{\gamma}_j$ are the normalized $k$-mode singular values of \textsf{L} such that $\sum_{j=1}^n\hat{\gamma}_j = 1$.
\end{definition}

The convention $0\ln{0}=0$ is used if $\hat{\gamma}_j=0$. The $k$-mode singular values of \textsf{L} can be computed from the matrix SVD of the $k$-mode unfolding $\textbf{L}_{(k)}$, which results in a $\mathscr{O}(n^{k+1})$ time complexity and a $\mathscr{O}(n^{k})$ space complexity, see Algorithm \ref{alg:1}. Since \textsf{L} is supersymmetric, any mode unfolding of \textsf{L} would yield the same unfolding matrix with the same singular values.  Moreover, the tensor entropy (\ref{eq:3}) can be viewed as a variation of von Neumann entropy defined for graphs, in which we regard $c\textbf{L}_{(k)}\textbf{L}_{(k)}^\top$ as the density matrix for some normalization constant $c$ \cite{Samuel2004,10.1093/comnet/cny028,Passerini2008TheVN}. In particular, when $k=2$, the tensor entropy is reduced to the classical von Neumann entropy for graphs. Like the eigenvalues of Laplacian matrices, the $k$-mode singular values play a significant role in identifying the structural patterns for uniform hypergraphs.

\begin{algorithm}[h]
\caption{Computing tensor entropy from SVD}
\label{alg:1}
\begin{algorithmic}[1]
\STATE{Given a $k$-uniform hypergraph \textsf{G} with $n$ vertices}\\
\STATE{Construct the adjacency tensor $\textsf{A}\in\mathbb{R}^{n\times n\times \dots \times n}$ from \\ \textsf{G} and compute the Laplacian tensor $\textsf{L} = \textsf{D} - \textsf{A}$ \\ where \textsf{D} is the degree tensor}\\
\STATE{Find the $k$-mode unfolding of \textsf{L}, i.e., $\textbf{L}_{(k)} = \texttt{reshape}(\textsf{L}, n, n^{k-1})$}\\
\STATE{Compute the economy-size matrix SVD of $\textbf{L}_{(k)}$, i.e., $\textbf{L}_{(k)}=\textbf{U}\textbf{S}\textbf{V}^\top$ and let $\{\gamma_j\}_{j=1}^{n}=\texttt{diag}(\textbf{S})$}\\
\STATE{Set $\hat{\gamma}_j=\frac{\gamma_j}{\sum_{i=1}^n\gamma_i}$ and compute $\textsc{S}  = -\sum_{j=1}^n\hat{\gamma}_j\ln{\hat{\gamma}_j}$}
\RETURN The tensor entropy \textsc{S} of \textsf{G}.
\end{algorithmic}
\end{algorithm}

\begin{lemma}\label{lem:1}
Suppose that \textsf{G} is a $k$-uniform hypergraph with $k\geq 3$. Then \textsf{L} has a $k$-mode singular value zero, with multiplicity $p$, if and only if \textsf{G} contains $p$ number of non-connected vertices.  
\end{lemma}
\begin{proof}
The result follows immediately from the definitions of Laplacian tensor and $k$-mode unfolding of \textsf{L}. 
\end{proof}

The multiplicity of the zero $k$-mode singular value can be used to determine the number of connected components of uniform hypergraphs.  Moreover, one can derive the lower bound of tensor entropy based on Lemma \ref{lem:1}.

\begin{proposition}\label{pro:1}
Suppose that \textsf{G} is a $k$-uniform hypergraph with $n$ vertices and nonempty hyperedge set \textbf{E} for $k\geq 3$. Then the minimum tensor entropy of \textsf{G} is given by 
\begin{equation}
\textsc{S}_{\min} = \ln{k}.
\end{equation}
\end{proposition}
\begin{proof}
Since \textsf{G} is a $k$-uniform hypergraph on $n$ vertices, the maximum multiplicity of the zero normalized $k$-mode singular value of $\textsf{L}$ is $n-k$ according to Lemma \ref{lem:1}. In addition, the other normalized $k$-mode singular values of  are necessarily $\frac{1}{k}$. Hence,  it is straightforward to show that $\textsc{S} _{\min} = \ln{k}$.
\end{proof}
Every $k$-uniform hypergraph can achieve the minimum tensor entropy $\ln{k}$. As the number of vertices contained in hyperedges increases, the lower limit of tensor entropy also increases. In the following, we present results about the upper limit of tensor entropy and its relation to regular uniform hypergraphs.

\begin{proposition}\label{pro:2}
Suppose that \textsf{G} is a $k$-uniform hypergraph with $n$ vertices for $k\geq 3$. Then the maximum tensor entropy of \textsf{G} occurs when it is a 1-regular uniform hypergraph, and is given by
\begin{equation}
\textsc{S} _{\max} = \ln{n}.
\end{equation}
\end{proposition}
\begin{proof}
Since \textsf{G} is a $k$-uniform hypergraph on $n$ vertices, the maximum tensor entropy occurs when the multiplicity of a normalized $k$-mode singular value of $\textsf{L}$ is $n$. Based on the definitions of Laplacian tensor and $k$-mode unfolding, for a $1$-regular uniform hypergraph, the number of nonzero elements in $\textbf{L}_{(k)}$ are fixed for $j$-th row with one entry
\begin{equation*}
(\textbf{L}_{(k)})_{j[1+\sum_{m=1}^{k-1}(j-1)n^{m-1}]} = 1 
\end{equation*} 
and $(k-1)!$ entries $-\frac{1}{2}$ for $j=1,2,\dots n$. Moreover, since all the hyperedges contain distinct vertices, the column indices of the nonzero entries are unique for $\textbf{L}_{(k)}$. Thus, $\textbf{L}_{(k)}\textbf{L}_{(k)}^\top$ is a diagonal matrix with equal diagonal elements, and the result follows immediately.
\end{proof}

\begin{proposition}\label{pro:5}
Suppose that \textsf{G} is a $k$-uniform hypergraph with $n$ vertices for $k\geq 3$. If  $\log_k{n}$ is an integer, then the maximum tensor entropy of \textsf{G} can be achieved when it is a $d$-regular uniform hypergraph for $1 \leq d \leq \log_k{n}$.
\end{proposition}
\begin{proof}
Suppose that $\log_k{n}$ is an integer. Any two hyperedges of a $d$-regular uniform hypergraph contain at least $k-1$ distinct vertices for $1 \leq d \leq \log_k{n}$. Similar to Proposition \ref{pro:2}, it can be shown that $\textbf{L}_{(k)}\textbf{L}_{(k)}^\top$ is a diagonal matrix with equal diagonal elements. Therefore, the result follows immediately.
\end{proof}

According to Proposition \ref{pro:2} and \ref{pro:5}, not all uniform hypergraphs with $n$ vertices can achieve the tensor entropy $\ln{n}$. However, when $\log_k{n}$ is an integer, one can utilize tensor entropy to measure the regularity of uniform hypergraphs. Moreover, if \textsf{G} contains $p$ number of non-connected vertices and $\log_k{(n-p)}$ is an integer, then $\textsc{S}_{\max} = \ln{(n-p)}$. Therefore, larger tensor entropy can be obtained with more connected components in this case. Next, we establish results regarding complete uniform hypergraphs.

\begin{figure*}[hbt!]
\centering
\tcbox[colback=mygray]{
\begin{tikzpicture}
\node[vertex,text=white,scale=0.7] (v1) {1};
\node[vertex,below of=v1,text=white,scale=0.7] (v2) {2};
\node[vertex,below of=v2,text=white,scale=0.7] (v3) {3};
\node[vertex,right of=v3,text=white,scale=0.7] (v4) {4};
\node[vertex,right of=v4,text=white,scale=0.7] (v5) {5};
\node[vertex,above of=v5,text=white,scale=0.7] (v6) {6};
\node[vertex,above of=v6,text=white,scale=0.7] (v7) {7};

\node[vertex,right of=v7, text=white,scale=0.7,xshift=5mm] (v8) {1};
\node[vertex,right of=v6,text=white,scale=0.7,xshift=5mm] (v9) {2};
\node[vertex,right of=v5,text=white,scale=0.7,xshift=5mm] (v10) {3};
\node[vertex,right of=v10,text=white,scale=0.7] (v11) {4};
\node[vertex,right of=v11,text=white,scale=0.7] (v12) {5};
\node[vertex,above of=v12,text=white,scale=0.7] (v13) {6};
\node[vertex,above of=v13,text=white,scale=0.7] (v14) {7};

\node[vertex,right of=v14, text=white,scale=0.7,xshift=5mm] (v15) {1};
\node[vertex,right of=v13,text=white,scale=0.7,xshift=5mm] (v16) {2};
\node[vertex,right of=v12,text=white,scale=0.7,xshift=5mm] (v17) {3};
\node[vertex,right of=v17,text=white,scale=0.7] (v18) {4};
\node[vertex,right of=v18,text=white,scale=0.7] (v19) {5};
\node[vertex,above of=v19,text=white,scale=0.7] (v20) {6};
\node[vertex,above of=v20,text=white,scale=0.7] (v21) {7};

\node[vertex,right of=v21, text=white,scale=0.7,xshift=5mm] (v22) {1};
\node[vertex,right of=v20,text=white,scale=0.7,xshift=5mm] (v23) {2};
\node[vertex,right of=v19,text=white,scale=0.7,xshift=5mm] (v24) {3};
\node[vertex,right of=v24,text=white,scale=0.7] (v25) {4};
\node[vertex,right of=v25,text=white,scale=0.7] (v26) {5};
\node[vertex,above of=v26,text=white,scale=0.7] (v27) {6};
\node[vertex,above of=v27,text=white,scale=0.7] (v28) {7};

\node[vertex,right of=v28, text=white,scale=0.7,xshift=5mm] (v29) {1};
\node[vertex,right of=v27,text=white,scale=0.7,xshift=5mm] (v30) {2};
\node[vertex,right of=v26,text=white,scale=0.7,xshift=5mm] (v31) {3};
\node[vertex,right of=v31,text=white,scale=0.7] (v32) {4};
\node[vertex,right of=v32,text=white,scale=0.7] (v33) {5};
\node[vertex,above of=v33,text=white,scale=0.7] (v34) {6};
\node[vertex,above of=v34,text=white,scale=0.7] (v35) {7};

\node[vertex,below of=v3, text=white,scale=0.7] (w1) {1};
\node[vertex,below of=w1,text=white,scale=0.7] (w2) {2};
\node[vertex,below of=w2,text=white,scale=0.7] (w3) {3};
\node[vertex,right of=w3,text=white,scale=0.7] (w4) {4};
\node[vertex,right of=w4,text=white,scale=0.7] (w5) {5};
\node[vertex,above of=w5,text=white,scale=0.7] (w6) {6};
\node[vertex,above of=w6,text=white,scale=0.7] (w7) {7};

\node[vertex,right of=w7, text=white,scale=0.7,xshift=5mm] (w8) {1};
\node[vertex,right of=w6,text=white,scale=0.7,xshift=5mm] (w9) {2};
\node[vertex,right of=w5,text=white,scale=0.7,xshift=5mm] (w10) {3};
\node[vertex,right of=w10,text=white,scale=0.7] (w11) {4};
\node[vertex,right of=w11,text=white,scale=0.7] (w12) {5};
\node[vertex,above of=w12,text=white,scale=0.7] (w13) {6};
\node[vertex,above of=w13,text=white,scale=0.7] (w14) {7};

\node[vertex,right of=w14, text=white,scale=0.7,xshift=5mm] (w15) {1};
\node[vertex,right of=w13,text=white,scale=0.7,xshift=5mm] (w16) {2};
\node[vertex,right of=w12,text=white,scale=0.7,xshift=5mm] (w17) {3};
\node[vertex,right of=w17,text=white,scale=0.7] (w18) {4};
\node[vertex,right of=w18,text=white,scale=0.7] (w19) {5};
\node[vertex,above of=w19,text=white,scale=0.7] (w20) {6};
\node[vertex,above of=w20,text=white,scale=0.7] (w21) {7};

\node[vertex,right of=w21, text=white,scale=0.7,xshift=5mm] (w22) {1};
\node[vertex,right of=w20,text=white,scale=0.7,xshift=5mm] (w23) {2};
\node[vertex,right of=w19,text=white,scale=0.7,xshift=5mm] (w24) {3};
\node[vertex,right of=w24,text=white,scale=0.7] (w25) {4};
\node[vertex,right of=w25,text=white,scale=0.7] (w26) {5};
\node[vertex,above of=w26,text=white,scale=0.7] (w27) {6};
\node[vertex,above of=w27,text=white,scale=0.7] (w28) {7};

\node[vertex,right of=w28, text=white,scale=0.7,xshift=5mm] (w29) {1};
\node[vertex,right of=w27,text=white,scale=0.7,xshift=5mm] (w30) {2};
\node[vertex,right of=w26,text=white,scale=0.7,xshift=5mm] (w31) {3};
\node[vertex,right of=w31,text=white,scale=0.7] (w32) {4};
\node[vertex,right of=w32,text=white,scale=0.7] (w33) {5};
\node[vertex,above of=w33,text=white,scale=0.7] (w34) {6};
\node[vertex,above of=w34,text=white,scale=0.7] (w35) {7};

\path [->,shorten >=4pt,shorten <=4pt, thick](v6) edge node[left] {} (v9);
\path [->,shorten >=4pt,shorten <=4pt,thick](v13) edge node[left] {} (v16);
\path [->,shorten >=4pt,shorten <=4pt,thick](v20) edge node[left] {} (v23);
\path [->,shorten >=4pt,shorten <=4pt,thick](v27) edge node[left] {} (v30);

\path [->,shorten >=4pt,shorten <=4pt,thick](w6) edge node[left] {} (w9);
\path [->,shorten >=4pt,shorten <=4pt,thick](w13) edge node[left] {} (w16);
\path [->,shorten >=4pt,shorten <=4pt,thick](w20) edge node[left] {} (w23);
\path [->,shorten >=4pt,shorten <=4pt,thick](w27) edge node[left] {} (w30);

\begin{pgfonlayer}{background}
\begin{scope}[transparency group,opacity=.9]
\draw[edge,color=orange] (v1) -- (v2) -- (v3);

\draw[edge,color=orange] (v8) -- (v9) -- (v10);
\draw[edge,color=red] (v12) -- (v13) -- (v14);

\draw[edge,color=orange] (v15) -- (v16) -- (v17);
\draw[edge,color=red] (v19) -- (v20) -- (v21);
\draw[edge,color=green] (v17) -- (v18) -- (v19);

\draw[edge,color=orange] (v22) -- (v23) -- (v24);
\draw[edge,color=red] (v26) -- (v27) -- (v28);
\draw[edge,color=green] (v24) -- (v25) -- (v26);
\draw[edge,color=blue] (v23) -- (v25) -- (v27);

\draw[edge,color=orange] (v29) -- (v30) -- (v31);
\draw[edge,color=red] (v33) -- (v34) -- (v35);
\draw[edge,color=green] (v31) -- (v32) -- (v33);
\draw[edge,color=blue] (v30) -- (v32) -- (v34);
\draw[edge,color=yellow] (v29) -- (v32) -- (v35);

\draw[edge,color=orange] (w1) -- (w2) -- (w3);

\draw[edge,color=orange] (w8) -- (w9) -- (w10);
\draw[edge,color=red,line width=12pt] (w9) -- (w10) -- (w11);

\draw[edge,color=orange] (w15) -- (w16) -- (w17);
\draw[edge,color=red,line width=12pt] (w16) -- (w17) -- (w18);
\draw[edge,color=green,line width=12pt] (w15) -- (w16) -- (w18);

\draw[edge,color=orange] (w22) -- (w23) -- (w24);
\draw[edge,color=red,line width=12pt] (w23) -- (w24) -- (w25);
\draw[edge,color=green,line width=12pt] (w22) -- (w23) -- (w25);
\draw[edge,color=blue,line width=10pt] (w22) -- (w24) -- (w25);

\draw[edge,color=orange] (w29) -- (w30) -- (w31);
\draw[edge,color=red,line width=12pt] (w30) -- (w31) -- (w32);
\draw[edge,color=green,line width=12pt] (w29) -- (w30) -- (w32);
\draw[edge,color=blue,line width=10pt] (w29) -- (w31) -- (w32);
\draw[edge,color=yellow,line width=8pt] (w31) -- (w32) -- (w33);

\end{scope}
\end{pgfonlayer}
\node[elabel,color=orange,fill,opacity=.6,fill opacity=.6,below of=w3, label=right:\(e_1\)]  (e1) {};
\node[elabel,right of=e1,color=red,fill,opacity=.6,fill opacity=.6,label=right:\(e_2\)]  (e2) {};
\node[elabel,right of=e2,color=green,fill,opacity=.6,fill opacity=.6,label=right:\(e_3\)]  (e3) {};
\node[elabel,right of=e3,color=blue,fill,opacity=.6,fill opacity=.6,label=right:\(e_4\)]  (e4) {};
\node[elabel,right of=e4,color=yellow,fill,opacity=.6,fill opacity=.6,label=right:\(e_5\)]  (e5) {};

\end{tikzpicture}}
\caption{\textbf{Tensor entropy maximization/minimization.} The top row describes the first five stages of the tensor entropy maximization evolution with a growing number of hyperedges in the order of  $e_1=\{1,2,3\}$, $e_2=\{5,6,7\}$, $e_3=\{3,4,5\}$, $e_4=\{2,4,6\}$ and $e_ 5=\{1,4,7\}$. The bottom row reports the first five stages of the tensor entropy minimization process with a growing number of hyperedges in the order of  $e_1=\{1,2,3\}$, $e_2=\{2,3,4\}$, $e_3=\{1,2,4\}$, $e_4=\{1,3,4\}$ and $e_ 5=\{3,4,5\}$.}
\label{fig:1}
\end{figure*}
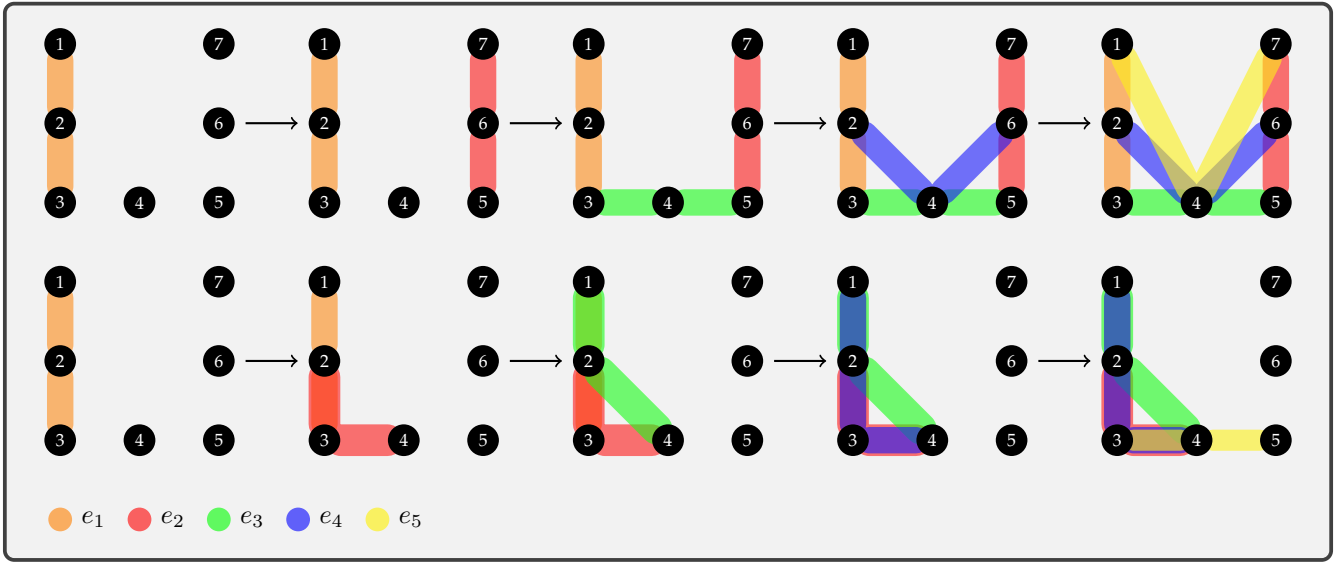

\begin{proposition}\label{pro:3}
Suppose that \textsf{G} is a complete $k$-uniform hypergraph with $n$ vertices for $k\geq 3$. Then the tensor entropy of \textsf{G} is given by
\begin{equation}\label{eq:12}
\begin{split}
\textsc{S} _{c} &= \frac{(1-n)\alpha}{(n-1)\alpha+\beta}\ln{\frac{\alpha}{(n-1)\alpha+\beta}} \\&- \frac{\beta}{(n-1)\alpha+\beta}\ln{\frac{\beta}{(n-1)\alpha+\beta}}
\end{split}
\end{equation} 
where,
\begin{small}
\begin{align}
\alpha &= (\frac{\Gamma(n)(\Gamma(n-k+1)+\Gamma(n))}{\Gamma(k)^2\Gamma(n-k+1)^2}-\frac{\Gamma(n-1)}{\Gamma(k)^2\Gamma(n-k)})^{\frac{1}{2}},\\
\beta  & = (\frac{\Gamma(n)(\Gamma(n-k+1)+\Gamma(n))}{\Gamma(k)^2\Gamma(n-k+1)^2}+\frac{(n-1)\Gamma(n-1)}{\Gamma(k)^2\Gamma(n-k)})^{\frac{1}{2}},
\end{align}
\end{small}
and $\Gamma(\cdot)$ is the Gamma function. 
\end{proposition}
\begin{proof}
Based on the definitions of Laplacian tensor and $k$-mode unfolding, the matrix $\textbf{L}_{(k)}\textbf{L}_{(k)}^\top\in\mathbb{R}^{n\times n}$ is given by
\begin{equation*}
\textbf{L}_{(k)}\textbf{L}_{(k)}^\top = \begin{bmatrix} 
d & \rho & \rho & \dots & \rho\\
\rho & d & \rho & \dots & \rho\\
\rho & \rho & d & \dots & \rho\\
\vdots & \vdots & \vdots & \ddots & \vdots\\
\rho & \rho & \rho &\dots & d
\end{bmatrix},
\end{equation*}
where, 
$$d =\binom {n-1}{k-1}^2+\binom {n-1}{k-1}\frac{1}{(k-1)!},\text{ and } \rho= \frac{T_{n-k}}{(k-1)!^2}.$$ Here $T_{m}$ are the $k$-simplex numbers (e.g., when $k=3$, $T_{m}$ are the triangular numbers). Moreover, the eigenvalues of $\textbf{L}_{(k)}\textbf{L}_{(k)}^\top$ are $d-\rho$ with multiplicity $n-1$ and $d+(n-1)\rho$ with multiplicity 1. Hence, the result follows immediately. We write all the expressions using the Gamma function for simplicity. 
\end{proof}

From Proposition \ref{pro:5} and \ref{pro:3}, for arbitrary $k$-uniform hypergraph with $n$ vertices and $k\geq 3$, $\textsc{S} _{c}$ could be smaller than the entropies of other $d$-regular hypergraphs, and $\textsc{S} _{c}\leq \textsc{S} _{\max}\leq \ln{n}$. Furthermore, it can be shown that when $n$ becomes large, $\textsc{S} _{\max}\approx \ln{n}$.
\begin{corollary}
Suppose that \textsf{G} is a complete $k$-uniform hypergraph with $n$ vertices for $k\geq 3$. Then the tensor entropy 
$\textsc{S}_c \rightarrow \ln{n} \text{ as } n\rightarrow \infty.$
\end{corollary}
\begin{proof}
As $n\rightarrow \infty$, $\frac{\Gamma(n)(\Gamma(n-k+1)+\Gamma(n))}{\Gamma(k)^2\Gamma(n-k+1)^2}\gg\frac{\Gamma(n-1)}{\Gamma(k)^2\Gamma(n-k)}$ for fixed $k$. Thus, $\alpha \approx \beta$, and the result follows immediately. 
\end{proof}

In section \ref{sec:3}, we will show evidence that the tensor entropy (\ref{eq:3}) is a measure of regularity for general uniform hypergraphs. Large tensor entropy is characterized by the large number of connected vertices, high uniformity of vertex degrees, short path lengths and high level of nontrivial symmetricity. The entropy is small for uniform hypergraphs with large \textit{cliques} and long path lengths, i.e., hypergraphs in which the vertices form  highly connected clusters. Tensor entropy is also related to the clustering coefficients of uniform hypergraphs in a very nuanced way.

\subsection{Numerical method via tensor trains}\label{sec:2.3}
In reality, hypergraphs like co-authorship networks and protein-protein interaction networks exist in a very large scale, and computing the tensor entropy using the economy-size matrix SVD could be computationally expensive. Klus et al. \cite{Klus_2018} exploited TTD to efficiently calculate the Moore-Penrose (MP) inverse of the matrix obtained from any chosen unfolding of a given tensor. TTD provides a good compromise between numerical stability and level of compression, and has an associated algebra that facilitates computations. We thus adapt the framework of Klus et al. for the computation of the tensor entropy, see Algorithm \ref{alg:2}. In step 2, we assume that the construction of adjacency and degree tensors in the TT-format can be achieved due to their simple structures. In step 4, the left- and right-orthonormalization algorithms can be found in \cite{Klus_2018}. The computation and memory complexities of Algorithm \ref{alg:2} are estimated as $\mathscr{O}(knr^3)$ and $\mathscr{O}(knr^2)$, respectively, where $r$ can be viewed as the ``average'' rank of the TT-ranks. Both complexities are much lower than those from Algorithm \ref{alg:1} when $r$ is small.

\begin{algorithm}[ht]
\caption{Computing tensor entropy from TTD}
\label{alg:2}
\begin{algorithmic}[1]
\STATE{Given a $k$-uniform hypergraph \textsf{G} with $n$ vertices}\\
\STATE{Construct the adjacency and degree tensors $\textsf{A},\textsf{D}\in\mathbb{R}^{n\times n\times \dots \times n}$ in the TT-format from  \textsf{G}}\\
\STATE{Compute the Laplacian tensor $\textsf{L} = \textsf{D} -\textsf{A}$ with core tensors $\textsf{X}^{(p)}$ and TT-ranks $\{R_0,R_1,\dots,R_k\}$ based \\ on  the tensor train summation operation }\\
\STATE{Left-orthonormalize the first $k-2$ cores and right-orthonormalize the last core of \textsf{L}}\\
\STATE{Compute the economy-size matrix SVD of $\bar{\textbf{X}}^{(k-1)}$, \\ i.e., $\bar{\textbf{X}}^{(k-1)}=\textbf{U}\textbf{S}\textbf{V}^\top$ for $s=\texttt{rank}(\textbf{S})$ and let $\{\gamma_j\}_{j=1}^{s}=\texttt{diag}(\textbf{S})$}\\
\STATE{Set $\hat{\gamma}_j=\frac{\gamma_j}{\sum_{i=1}^s\gamma_i}$ and compute $\textsc{S}  = -\sum_{j=1}^s\hat{\gamma}_j\ln{\hat{\gamma}_j}$}
\RETURN The tensor entropy \textsc{S}  of \textsf{G}.
\end{algorithmic}
\end{algorithm}

\section{Experiments}\label{sec:3}
All the numerical examples presented were performed on a Linux machine with 8 GB RAM and a 2.4 GHz Intel Core i5 processor in MATLAB 2018b. The last example (section \ref{sec:train}) also used the MATLAB TT-Toolbox by Oseledets et al. \cite{tttoolbox}.

\subsection{Hyperedge growth model}\label{sec:3.1}
We consider the case where the number of vertices is fixed and new hyperedges are iteratively added to the uniform hypergraph. Figure \ref{fig:1}  presents the hyperedge growth evolution of a 3-uniform hypergraph with 7 vertices, and it describes the tensor entropy maximization and minimization evolutions. In addition to plotting the two entropy trajectories, we also compute some statistics of the structural properties including average shortest path length, index of dispersion of the degree distribution and average clustering coefficient of the hypergraphs during the two evolutions, see Figure \ref{fig:1.5}. If the two vertices are disconnected, we set the distance between them to be 4  for  the  purpose  of visualization in Figure \ref{fig:1.5}B.
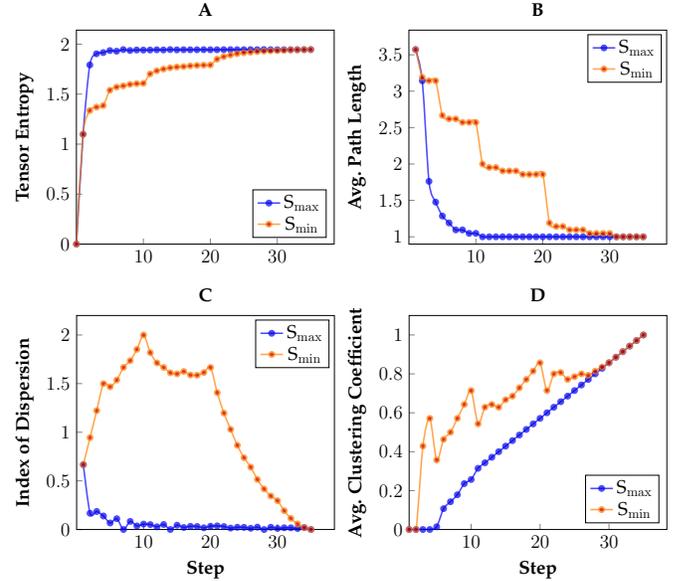
\begin{figure}[ht]
\centering
\begin{tikzpicture}[scale=0.5][font=\Large]
\begin{axis}[ylabel near ticks,
xlabel=\textbf{}, xtick={10,20,30}, ylabel=\textbf{Tensor Entropy}, xmin=0, ymin=0,title=\Large\textbf{A},legend pos = south east,xtick pos=left,ytick pos=left]
\addplot+[smooth,very thick,mark=*,opacity=.7] plot coordinates
{ (0,0) (1,1.0986) (2,1.7918) (3,1.9037) (4, 1.9157) (5,1.9359) (6,1.9297) (7,1.9456) (8,1.9366) (9,1.9417) (10,1.9408) (11,1.9414) (12, 1.9437) (13,1.9420) (14,1.9457) (15,1.9430) (16,1.9444) (17,1.9439) (18,1.9440) (19,1.9448) (20,1.9441) (21,1.9440) (22,1.9446) (23,1.9450) (24,1.9447) (25,1.9448) (26,1.9452) (27,1.9447) (28,1.9456) (29,1.9449) (30,1.9444) (31,1.9443) (32,1.9443) (33,1.9446) (34,1.9450) (35,1.9456)};
\addlegendentry{$\textsc{S}_{\max}$}
\addplot+[smooth,very thick,mark=*,opacity=.7,orange] plot coordinates
{ (0,0) (1,1.0986) (2,1.3358) (3,1.3701) (4, 1.3855) (5,1.5393) (6,1.5707) (7,1.5826) (8,1.5976) (9,1.6050) (10,1.6087) (11,1.7023) (12,1.7334) (13,1.7506) (14,1.7637) (15,1.7714) (16,1.7757) (17,1.7829) (18,1.7873) (19,1.7900) (20,1.7913) (21,1.8502) (22,1.8749) (23,1.8914) (24,1.9040) (25,1.9130) (26,1.9193) (27,1.9257) (28,1.9303) (29,1.9335) (30,1.9355) (31,1.9393) (32,1.9421) (33,1.9439) (34,1.9450) (35,1.9456)};
\addlegendentry{$\textsc{S}_{\min}$}
\end{axis}
\end{tikzpicture}
\begin{tikzpicture}[scale=0.5][font=\Large]
\begin{axis}[ylabel near ticks,
xlabel=\textbf{}, xtick={10,20,30}, ylabel=\textbf{Avg. Path Length}, xmin=0, ymin=0.9,title=\textbf{B},xtick pos=left,ytick pos=left]
\addplot+[smooth,very thick,mark=*,opacity=.7] plot coordinates
{ (1,3.5714) (2,3.1429) (3,1.7619) (4, 1.4762) (5,1.2857) (6,1.1905) (7,1.0952) (8,1.0952) (9,1.0476) (10,1.0476) (11,1) (12,1) (13,1) (14,1) (15,1) (16,1) (17,1) (18,1) (19,1) (20,1) (21,1) (22,1) (23,1) (24,1) (25,1) (26,1) (27,1) (28,1) (29,1) (30,1) (31,1) (32,1) (33,1) (34,1) (35,1)};
\addlegendentry{$\textsc{S}_{\max}$}
\addplot+[smooth,very thick,mark=*,opacity=.7,orange] plot coordinates
{ (1, 3.5714) (2,3.1905) (3,3.1429) (4,3.1429) (5,2.6667) (6,2.6190) (7,2.6190) (8, 2.5714) (9, 2.5714) (10, 2.5714) (11,2.0000) (12,1.9524) (13,1.9524) (14,1.9048) (15,1.9048) (16,1.9048) (17,1.8571) (18,1.8571) (19,1.8571) (20,1.8571) (21,1.1905) (22,1.1429) (23,1.1429) (24,1.0952) (25,1.0952) (26,1.0952) (27,1.0476) (28,1.0476) (29,1.0476) (30,1.0476) (31,1) (32,1) (33,1) (34,1) (35,1)};
\addlegendentry{$\textsc{S}_{\min}$}
\end{axis}
\end{tikzpicture}

\begin{tikzpicture}[scale=0.5][font=\Large]
\begin{axis}[ylabel near ticks,
xlabel=\textbf{Step},ylabel=\textbf{Index of Dispersion}, xmin=0, ymin=0,xtick={10,20,30},title=\textbf{C},xtick pos=left,ytick pos=left]
\addplot+[smooth,very thick,mark=*,opacity=.7] coordinates
{ (1,0.6667) (2,0.1667) (3,0.1852) (4, 0.1389) (5,0.0667) (6,0.1111) (7,0) (8,0.0833) (9,0.0370) (10,0.0556) (11,0.0505) (12,0.0278) (13,0.0513) (14,0) (15,0.0444) (16,0.0208) (17,0.0327) (18,0.0309) (19,0.0175) (20,0.0333) (21,0.0370) (22,0.0303) (23,0.0145) (24,0.0231) (25,0.0222) (26,0.0128) (27,0.0247) (28,0) (29,0.0230) (30,0.0111) (31,0.0179) (32,0.0174) (33,0.0101) (34,0.0196) (35,0)};
\addlegendentry{$\textsc{S}_{\max}$}
\addplot+[smooth,very thick,mark=*,opacity=.7,orange] coordinates
{ (1,0.6667) (2,0.9444) (3,1.2222) (4, 1.5000) (5,1.4667) (6,1.5370) (7,1.6667) (8,1.7361) (9,1.8519) (10,2) (11,1.8182) (12,1.7130) (13,1.6667) (14,1.6111) (15,1.6000) (16,1.6250) (17,1.5882) (18,1.5864) (19,1.6140) (20,1.6667) (21,1.4074) (22,1.1970) (23,1.0290) (24,0.8657) (25,0.7378) (26,0.6410) (27,0.5144) (28,0.4167) (29,0.3448) (30,0.2963) (31,0.1935) (32,0.1146) (33,0.0572) (34,0.0196) (35,0) };
\addlegendentry{$\textsc{S}_{\min}$}
\end{axis}
\end{tikzpicture}
\begin{tikzpicture}[scale=0.5][font=\Large]
\begin{axis}[ylabel near ticks,
xlabel=\textbf{Step},ylabel=\textbf{Avg. Clustering Coefficient}, xmin=1, ymin=0,xtick={10,20,30}, legend pos = south east,title=\textbf{D},xtick pos=left,ytick pos=left]
\addplot+[smooth,very thick,mark=*,opacity=.7] coordinates
{ (1,0) (2,0) (3,0) (4, 0) (5,0.0143) (6,0.1071) (7,0.1429) (8,0.1786) (9,0.2357) (10,0.2571) (11,0.3143) (12,0.3429) (13,0.3714) (14,0.4000) (15,0.4286) (16,0.4571) (17,0.4857) (18, 0.5143) (19,0.5429) (20,0.5714) (21,0.6) (22,0.6286) (23,0.6571) (24,0.6857) (25,0.7143) (26,0.7429) (27,0.7714) (28,0.8000) (29,0.8286) (30,0.8571) (31,0.8857) (32,0.9143) (33,0.9429) (34,0.9714) (35,1.0000)};
\addlegendentry{$\textsc{S}_{\max}$}
\addplot+[smooth,very thick,mark=*,opacity=.7,orange] coordinates
{ (1,0) (2,0) (3,0.4286) (4,0.5714) (5,0.3571) (6,0.4632) (7,0.5) (8,0.5714) (9,0.6429) (10,0.7143) (11,0.5429) (12,0.6286) (13,0.6429) (14,0.6286)  (15,0.6671) (16,0.6857) (17,0.7286) (18,0.7714) (19,0.8143) (20,0.8571) (21,0.7143) (22,0.8000) (23,0.8071) (24,0.7714) (25,0.7857) (26,0.8) (27,0.7929) (28,0.8143) (29,0.8357) (30,0.8571) (31,0.8857)  (32,0.9143) (33,0.9429) (34,0.9714) (35,1)};
\addlegendentry{$\textsc{S}_{\min}$}
\end{axis}
\end{tikzpicture}
\caption{\textbf{Hyperedge growth model features.}  (A), (B), (C) and (D) Trajectories of tensor entropy, average path length, index of dispersion and average clustering coefficient with respect to the hyperedge adding steps. }
\label{fig:1.5}
\end{figure}

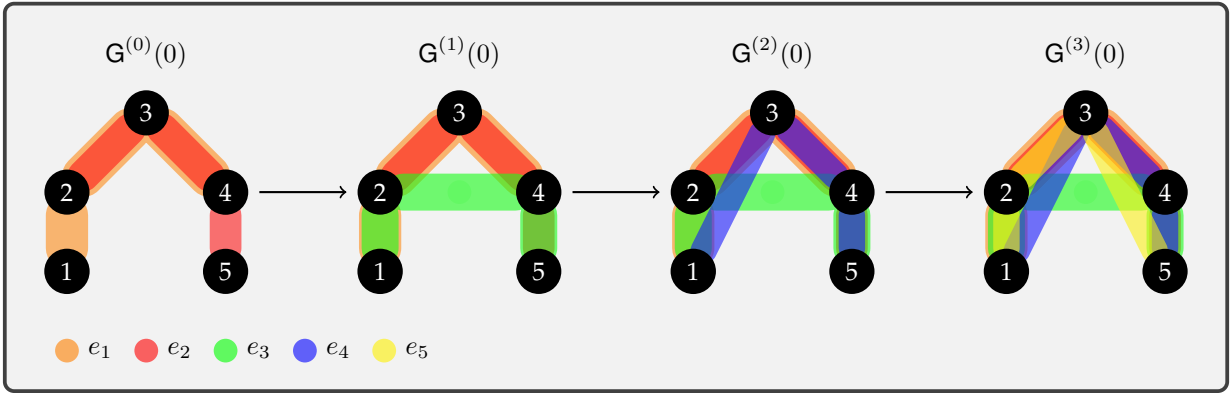
\begin{figure*}[hbt!]
\centering
\tcbox[colback=mygray]{
 \begin{tikzpicture}
\node[vertex,text=white] (v1) {2};
\node[vertex,below of=v1,text=white] (v2) {1};
\node[vertex,right of=v1,text=white,mygray] (v3) {};
\node[vertex,above of=v3,text=white] (v4) {3};
\node[vertex,right of=v3,text=white] (v5) {4};
\node[vertex,below of=v5,text=white] (v6) {5};
\node[above of=v4,yshift=-2mm]  (A) {$\textsf{G}^{(0)}(0)$};

\node[vertex,right of=v5,text=white, xshift=10mm] (v7) {2};
\node[vertex,below of=v7,text=white] (v8) {1};
\node[vertex,right of=v7,text=white,green,fill opacity=.1] (v9) {};
\node[vertex,above of=v9,text=white] (v10) {3};
\node[vertex,right of=v9,text=white] (v11) {4};
\node[vertex,below of=v11,text=white] (v12) {5};
\node[above of=v10,yshift=-2mm]  (B) {$\textsf{G}^{(1)}(0)$};

\node[vertex,right of=v11,text=white, xshift=10mm] (v13) {2};
\node[vertex,below of=v13,text=white] (v14) {1};
\node[vertex,right of=v13,text=white,green,fill opacity=.1] (v15) {};
\node[vertex,above of=v15,text=white] (v16) {3};
\node[vertex,right of=v15,text=white] (v17) {4};
\node[vertex,below of=v17,text=white] (v18) {5};
\node[above of=v16,yshift=-2mm]  (C) {$\textsf{G}^{(2)}(0)$};

\node[vertex,right of=v17,text=white, xshift=10mm] (v19) {2};
\node[vertex,below of=v19,text=white] (v20) {1};
\node[vertex,right of=v19,text=white,green,fill opacity=.1] (v21) {};
\node[vertex,above of=v21,text=white] (v22) {3};
\node[vertex,right of=v21,text=white] (v23) {4};
\node[vertex,below of=v23,text=white] (v24) {5};
\node[above of=v22,yshift=-2mm]  (D) {$\textsf{G}^{(3)}(0)$};

\path [->,shorten >=4pt,shorten <=4pt, thick](v5) edge node[left] {} (v7);
\path [->,shorten >=4pt,shorten <=4pt, thick](v11) edge node[left] {} (v13);
\path [->,shorten >=4pt,shorten <=4pt, thick](v17) edge node[left] {} (v19);

\begin{pgfonlayer}{background}
\begin{scope}[transparency group,opacity=.9]
\draw[edge,color=orange,line width=16pt] (v2) -- (v1) -- (v4) -- (v5);
\draw[edge,color=red,line width=12pt] (v1) -- (v4) -- (v5) -- (v6);

\draw[edge,color=orange,line width=16pt] (v8) -- (v7) -- (v10) -- (v11);
\draw[edge,color=red,line width=12pt] (v7) -- (v10) -- (v11) -- (v12);
\draw[edge,color=green, line width=14pt] (v8) -- (v7) -- (v11) -- (v12);

\draw[edge,color=orange,line width=16pt] (v14) -- (v13) -- (v16) -- (v17);
\draw[edge,color=red,line width=12pt] (v13) -- (v16) -- (v17) -- (v18);
\draw[edge,color=green, line width=14pt] (v14) -- (v13) -- (v17) -- (v18);
\draw[edge,color=blue] (v14) -- (v16) -- (v17) -- (v18);

\draw[edge,color=orange,line width=16pt] (v20) -- (v19) -- (v22) -- (v23);
\draw[edge,color=red,line width=12pt] (v19) -- (v22) -- (v23) -- (v24);
\draw[edge,color=green, line width=14pt] (v20) -- (v19) -- (v23) -- (v24);
\draw[edge,color=blue] (v20) -- (v22) -- (v23) -- (v24);
\draw[edge,color=yellow] (v20) -- (v19) -- (v22) -- (v24);
\end{scope}
\end{pgfonlayer}
\node[elabel,color=orange,fill,opacity=.6,fill opacity=.6,below of=v2, label=right:\(e_1\)]  (e1) {};
\node[elabel,right of=e1,color=red,fill,opacity=.6,fill opacity=.6,label=right:\(e_2\)]  (e2) {};
\node[elabel,right of=e2,color=green,fill,opacity=.6,fill opacity=.6,label=right:\(e_3\)]  (e3) {};
\node[elabel,right of=e3,color=blue,fill,opacity=.6,fill opacity=.6,label=right:\(e_4\)]  (e4) {};
\node[elabel,right of=e4,color=yellow,fill,opacity=.6,fill opacity=.6,label=right:\(e_5\)]  (e5) {};
\end{tikzpicture}
}
\caption{\textbf{Initial hypergraphs' structures for different $\boldsymbol{q}$.} The plot describes the \textit{cliques}' formation in the first five vertices of the uniform hypergraph with the rewiring probability zero, in which $e_1=\{1,2,3,4\}$, $e_2=\{2,3,4,5\}$, $e_3=\{1,2,4,5\}$, $e_4=\{1,3,4,5\}$ and $e_5=\{1,2,3,5\}$. The rest have the same patterns in every five vertices for a corresponding $q$.}
\label{fig:3}
\end{figure*}

Let's denote the hypergraphs that achieve maximum (or minimum) tensor entropy at step $j$ as $\textsf{G}_{\max}^{(j)}$ (or $\textsf{G}_{\min}^{(j)}$) for $j = 1,2,\dots,35$. Similar to maximizing graph entropy, maximizing the tensor entropy will first connect all the vertices and then prefer to choose lower degree vertices with larger average geodesic distances, i.e., finding the geodesic distances between each pair in the triples and taking the mean, see Figure \ref{fig:1} and \ref{fig:1.5}B. The average geodesic distances may lose importance if one wants to predict the next step as the hypergraph becomes complex. Moreover, the vertices of $\textsf{G}_{\max}^{(j)}$ tend to have ``almost equal'' or equal degree which leads to a low index of dispersion, see Figure \ref{fig:1.5}C. In particular, $\textsf{G}_{\max}^{(j)}$ are the $\frac{k}{n}j$-regular hypergraphs for the early stages of the evolution, i.e., $j=7,14$. However, as the hypergraph becomes dense, it is possible that $\textsf{G}_{\max}^{(j)}$ will miss the regularity, i.e., $j=21$. Additionally, the average clustering coefficient, in general, grows with increase of hyperedges, but the average growth rate for $\textsf{G}_{\max}^{(j)}$ is lower than that for $\textsf{G}_{\min}^{(j)}$, see Figure \ref{fig:1.5}D. Furthermore, nontrivial symmetricity plays a role in maximizing the tensor entropy. For example, in $\textsf{G}_{\max}^{(3)}$, the vertices $\{1,2,6\}$ and $\{2,4,6\}$ have the same average geodesic distances (both are equal to $\frac{7}{3}$), and the maximized tensor entropy returns the more symmetric $\textsf{G}_{\max}^{(4)}$. We also find that candidate hyperedges that intersect more existing hyperedges would return higher tensor entropy, which also explains the above example.

However, there exists one huge disparity between the von Neumann graph entropy and the tensor entropy. The tensor entropy can temporarily decrease during the maximizing process as seen in Figure \ref{fig:1.5}A. We observe that once the maximization evolution reaches some regularity or high level of nontrivial symmetricity, and the next step breaks such regularity or symmetricity, the tensor entropy will decrease. In other words, for these highly regular or highly symmetric $\textsf{G}_{\max}^{(j)}$, the corresponding tensor entropies $\textsc{S}^{(j)}_{\max}$ achieve local maxima. On the other hand, $\textsc{S}_{\min}^{(j)}$, the tensor entropies of $\textsf{G}_{\min}^{(j)}$, are similar to the von Neumann graph entropy. Minimizing the tensor entropy would result in the formations of complete sub-hypergraphs (\textit{cliques}), see Figure \ref{fig:1}.  We can detect large jumps and drops in the next steps after completions of the sub-hypergraphs in Figure \ref{fig:1.5}A and \ref{fig:1.5}B, respectively. In order to make the discoveries more convincing, we repeated the same processes for $k$-uniform hypergraphs with different number of vertices and values of $k$, and observed similar results.

\subsection{The Watts-Strogatz model}
We perform an experiment on a synthetic random uniform hypergraph \textsf{G} with $n=100$ and $k=4$. Similar to the Watts-Strogatz graph, the initial hypergraph is 2-regular with lattice structure. Let $q$ be the number of hyperedges added to the hypergraph in order to form \textit{cliques} in every five vertices, and $p$ be the rewiring probability of hyperedges, see Figure \ref{fig:3}. Then $\textsf{G}^{(q)}(p)$ denotes the random uniform hypergraphs generated by the rewiring probability $p$ for different $q$. Particularly, when $q=3$, the tensor entropy $\textsc{S}^{(3)}(0)=4.5527$, the average clustering coefficient $C_{\text{avg}}^{(3)}(0) = 0.7571$ and the average path length $L_{\text{avg}}^{(3)}(0)=7.0606$. The goal of the experiment is to explore the relations between the tensor entropy, the average clustering coefficient and path length with increasing the hyperedge rewiring probability $p$ for different $q$. We also calculate the small world coefficient for the random hypergraphs, denoted by $\sigma^{(q)}(p)$. For each $\textsf{G}^{(q)}(p)$, we compute its tensor entropy, average clustering coefficient, average path length and small world coefficient 10 times and take the means for $q=2,3$.

\usepgfplotslibrary{groupplots}
\begin{figure}[ht]
\centering
\begin{tikzpicture}[scale=0.5][font=\large]
\begin{groupplot}[group style={group size= 2 by 2, ,vertical sep=5em, horizontal sep=5em}]
 \nextgroupplot[ylabel near ticks,
xlabel=\textbf{Rewiring Probability},ylabel=\textbf{Tensor Entropy}, xmode=log,xtick={1e-3,1e-2,1e-1,1}, ymin=4.51,title=\textbf{A},legend pos = south west, ytick={4.52,4.53,4.54,4.55,4.56},xtick pos=left,ytick pos=left]
\addplot+[smooth,very thick,mark=*,opacity=.7,blue] coordinates
{ (0.001,4.5623)  (0.002,4.5623)  (0.003, 4.5627) (0.004,4.5623) (0.005,4.5616) (0.006,4.5615) (0.007,4.5616) (0.008, 4.5594) (0.009,4.5614) (0.01,4.5610) (0.02,4.5604) (0.03,4.5567) (0.04,4.5558) (0.05,4.5530) (0.06,4.5540) (0.07,4.5504) (0.08,4.5488) (0.09,4.5463) (0.1,4.5435) (0.2,4.5307) (0.3,4.5191) (0.4,4.5129)};
\addlegendentry{$q=2$}
\addplot+[smooth,very thick,mark=*,opacity=.7, orange] coordinates
{ (0.001,4.5528) (0.002,4.5525)  (0.003, 4.5523) (0.004,4.5524) (0.005,4.5527) (0.006,4.5512) (0.007,4.5515) (0.008, 4.5515) (0.009,4.5513) (0.01,4.5507) (0.02,4.5518) (0.03,4.5505) (0.04,4.5475) (0.05,4.5502) (0.06,4.5465) (0.07,4.5469) (0.08,4.5438) (0.09,4.5453) (0.1,4.5423) (0.2,4.5364) (0.3,4.5276) (0.4,4.5266)};
\addlegendentry{$q=3$}

 \nextgroupplot[ytick={0.2,0.4,0.6,0.8,1}, ylabel near ticks,
xlabel=\textbf{Rewiring Probability},ylabel=\textbf{Normalized Small World Coeff.}, xmode=log,xtick={1e-3,1e-2,1e-1,1}, ymin=0,title=\textbf{B},legend pos = south west,xtick pos=left,ytick pos=left]
\addplot+[smooth,very thick,mark=*,opacity=.7,blue] coordinates
{(0.001,0.4012) (0.002,0.3993) (0.003,0.3954) (0.004, 0.3963) (0.005, 0.4400) (0.006,0.4331) (0.007,0.4090) (0.008, 0.4763) (0.009,0.4428) (0.01,0.4273) (0.02,0.4764) (0.03,0.5343) (0.04,0.5281) (0.05,0.5485) (0.06,0.5301) (0.07,0.5091) (0.08,0.4992) (0.09, 0.5476) (0.1,0.4784) (0.2,0.2740) (0.3,0.1568) (0.4,0.0525)};
\addlegendentry{$q=2$}
\addplot+[smooth,very thick,mark=*,opacity=.7,orange] coordinates
{ (0.001, 0.7007) (0.002,0.7399) (0.003,0.7227) (0.004,0.7352) (0.005,0.7436) (0.006,0.8804) (0.007,0.8453) (0.008, 0.8532) (0.009,0.8790) (0.01,0.9094) (0.02,0.8761) (0.03,1.0000) (0.04,0.9910) (0.05,0.9954) (0.06,0.9881) (0.07,0.9885) (0.08,0.9092) (0.09,0.9005) (0.1,0.8481 ) (0.2,0.5380) (0.3,0.3108) (0.4,0.1251)};
\addlegendentry{$q=3$}

 \nextgroupplot[ ylabel near ticks,
xlabel=\textbf{Rewiring Probability},ylabel=\textbf{Ratio}, xmode=log,xtick={1e-3,1e-2,1e-1,1}, ymin=0,title=\textbf{C},legend pos = south west,xtick pos=left,ytick pos=left,ytick={0.20,0.40,0.60,0.80,1.00}]
\addplot+[smooth,very thick,mark=*,opacity=.7] coordinates
{(0.001,0.9937) (0.002,0.9857) (0.003,0.9834) (0.004,0.9705) (0.005,0.9832) (0.006,0.9468) (0.007,0.9522) (0.008,0.9534) (0.009,0.9493) (0.01,0.9378) (0.02,0.8745) (0.03,0.8495) (0.04,0.8193) (0.05,0.7964) (0.06,0.7148) (0.07,0.7249) (0.08,0.6102) (0.09,0.6021) (0.1,0.5375) (0.2,0.2980) (0.3,0.1691) (0.4,0.0594)};
\addlegendentry{Avg. Clustering Coeff.}
\addplot+[smooth,very thick,mark=*,opacity=.7,orange] coordinates
{(0.001,0.9815) (0.002,0.9377) (0.003,0.9441) (0.004,0.9192) (0.005,0.9328) (0.006,0.7767) (0.007,0.7923) (0.008, 0.7236) (0.009,0.7676) (0.01,0.7261) (0.02,0.7081) (0.03,0.5903) (0.04,0.5676) (0.05,0.5578) (0.06,0.4991) (0.07,0.5063) (0.08,0.4630) (0.09,0.4558) (0.1,0.4367) (0.2,0.3795) (0.3, 0.3474) (0.4,0.3241)};
\addlegendentry{Avg. Path Length}

 \nextgroupplot[ylabel near ticks,
xlabel=\textbf{Ratio},ylabel=\textbf{Tensor Entropy}, xmin=0, ymin=4.52,title=\textbf{D},legend pos = south east, ytick={4.53,4.54,4.55,4.56},xtick align=inside,xtick pos=left,ytick pos=left]
\addplot + [only marks, opacity=.7, blue] coordinates
{ (0.9937,4.5528) (0.9857,4.5525)  (0.9834, 4.5523) (0.9705,4.5524) (0.9832,4.5527) (0.9468,4.5512) (0.9522,4.5515) (0.9534, 4.5515) (0.9493,4.5513) (0.9378,4.5507) (0.8745,4.5518) (0.8495,4.5505) (0.8193,4.5475) (0.7964,4.5502) (0.7148,4.5465) (0.7249,4.5469) (0.6102,4.5438) (0.6021,4.5453) (0.5375,4.5423) (0.2980,4.5364) (0.1691,4.5276) (0.0594,4.5266)};
\addlegendentry{Avg. Clustering Coeff.}
\addplot + [only marks, opacity=.7, orange] coordinates
{ (0.9815,4.5528) (0.9377,4.5525)  (0.9441, 4.5523) (0.9192,4.5524) (0.9328,4.5527) (0.7767,4.5512) (0.7923,4.5515) (0.7236, 4.5515) (0.7676,4.5513) (0.7261,4.5507) (0.7081,4.5518) (0.5903,4.5505) (0.5676,4.5475) (0.5578,4.5502) (0.4991,4.5465) (0.5063,4.5469) (0.4630,4.5438) (0.4558,4.5453) (0.4367,4.5423) (0.3795,4.5364) (0.3474,4.5276) (0.3241,4.5266)};
\addlegendentry{Avg. Path Length}
\end{groupplot}
\end{tikzpicture}

\caption{\textbf{The Watts-Strogatz model features.} (A) Tensor entropies of random uniform hypergraphs with different rewiring probabilities for different $q$. (B) Normalized small world coefficients of random uniform hypergraphs with different rewiring probabilities for different $q$. (C) Ratios $C_{\text{avg}}^{(3)}(p)/C_{\text{avg}}^{(3)}(0)$ and $L_{\text{avg}}^{(3)}(p)/L_{\text{avg}}^{(3)}(0)$ of random uniform hypergraphs with different rewiring probabilities for $q=3$. (D) Scatter plot between the tensor entropy and the two ratios from (C).}
\label{fig:3.5}
\end{figure}
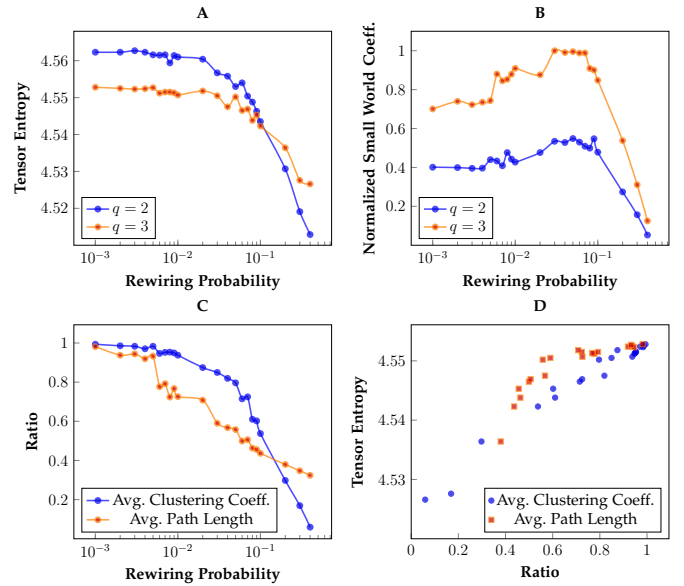

\begin{figure*}[ht]
\centering
\begin{tikzpicture}[scale=0.66][font=\large]
\begin{groupplot}[group style={group size=3 by 1, horizontal sep=5em}]
\nextgroupplot[ybar, enlargelimits=0.15, bar width = 5pt, ymax=2200, enlarge y limits=0.05, legend columns=-1,
    ylabel={\textbf{Num. of Contacts}}, 
    symbolic x coords={A,B,C,D,E,F,G,H,I},xtick pos=left,ytick pos=left,
    xtick=data, ylabel near ticks, xticklabels={1,2,3,4,5,6,7,8,9}, title=\textbf{A}, xlabel=\textbf{Hour}
    ]
\addplot+[blue, mark options={draw=blue,fill=blue}, opacity=.7] coordinates {(A,800) (B,1886) (C,1618) (D,1677) (E,1140) (F,1407) (G,981) (H,1794) (I,611)};
\addlegendentry{$k=2$}
\addplot[fill=orange,opacity=.7] coordinates {(A,104) (B, 520) (C,286) (D,509) (E,332) (F,269) (G,202) (H,353) (I,51)};
\addlegendentry{$k=3$}
\nextgroupplot[ybar, enlargelimits=0.15, bar width = 5pt, ymax=280, enlarge y limits=0.025, legend columns=-1, ymin=0,
    ylabel={\textbf{Num. of People Involved}}, 
    symbolic x coords={A,B,C,D,E,F,G,H,I},xtick pos=left,ytick pos=left,
    xtick=data, ylabel near ticks, xticklabels={1,2,3,4,5,6,7,8,9}, title=\textbf{B}, xlabel=\textbf{Hour}
    ]
\addplot[fill=blue, opacity=.7] coordinates {(A,228) (B,231) (C,233) (D,220) (E,118) (F,217) (G,215) (H,232) (I,229)};
\addlegendentry{$k=2$}
\addplot[fill=orange,opacity=.7] coordinates {(A,143) (B, 207) (C,188) (D,155) (E,114) (F,154) (G,134) (H,194) (I,76)};
\addlegendentry{$k=3$}
\nextgroupplot[xlabel={\textbf{Hour}}, ylabel={\textbf{Tensor Entropy}}, ylabel near ticks, title={\textbf{C}}, ymin=4.1, legend pos=south west, ytick={4.2,4.4,4.6,4.8,5,5.2,5.4}, xlabel near ticks,xtick pos=left,ytick pos=left,symbolic x coords={A,B,C,D,E,F,G,H,I}, xtick=data, xticklabels={1,2,3,4,5,6,7,8,9},xtick align=outside]
\addplot+[smooth, very thick,mark=*,blue, mark options={draw=blue,fill=blue}, opacity=.7] coordinates {(A,5.2226) (B,5.2727) (C,5.2537) (D,5.0540) (E,4.6473) (F,5.0985) (G,5.1163) (H,5.2486) (I,5.1662)};
\addlegendentry{$k=2$}
\addplot+[smooth, very thick,mark=*,orange, mark options={draw=orange,fill=orange}, opacity=.7] coordinates {(A,4.7736) (B,5.0847) (C,5.0539) (D,4.8104) (E,4.5688) (F,4.7933) (G,4.5510) (H,5.0184) (I,4.1486)};
\addlegendentry{$k=3$}
\end{groupplot}
\end{tikzpicture}
\caption{\textbf{Primary school contact features.} (A) Number of the two-person and three-person contacts amongst the children and teachers every one hour of a day. (B) Number of children and teachers involved every one hour of a day in the two-person and three-person contacts, respectively. (C) Trajectories of the von Neumann entropy and the tensor entropy for the two-person and three-person contacts of a day.}
\label{fig:5}
\end{figure*}
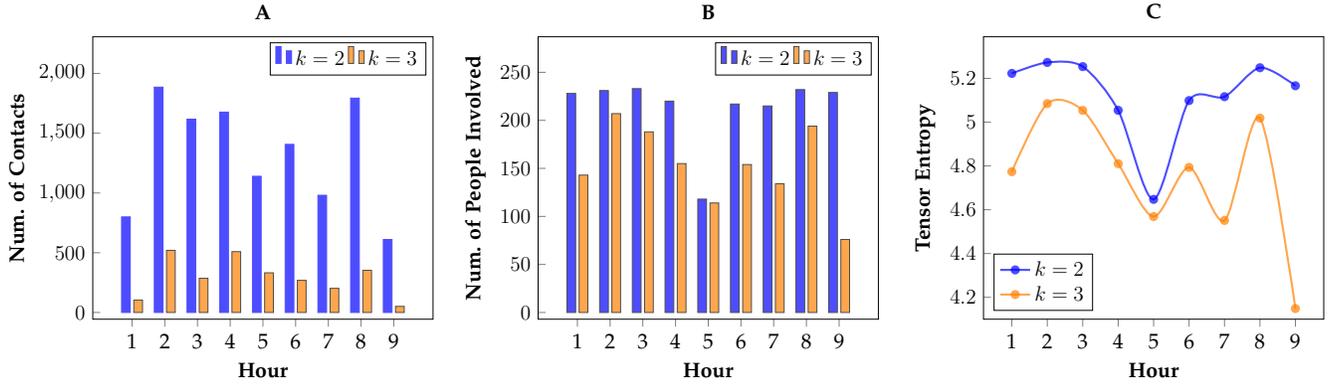

The results are shown in Figure \ref{fig:3.5}. In general, with increasing rewiring probability $p$, the tensor entropy decays for both $q=2$ and $3$, see Figure \ref{fig:3.5}A. Initially, the tensor entropies for $q=2$ are higher than those for $q=3$, which implies that lower average clustering coefficient yields larger tensor entropy at the same probability ($\sim$0.50 and $\sim$0.73, respectively). However, we see a strictly positive correlation between the tensor entropy and the average clustering coefficient as $p$ increases, see Figure \ref{fig:3.5}D. After $p=0.1$,  the tensor entropies for $q=2$ decays faster than that for $q=3$, indicating that higher average path length returns lower tensor entropy ($\sim$2.75 and $\sim$2.66, respectively). Moreover, the curves of the small world coefficient in Figure \ref{fig:3.5}B are very similar to the one in the Watts-Strogatz graph, in which it grows gradually for $p<0.1$ and decreases quickly after $p>0.1$. When the rewiring probability $p$ is between 0.03 and 0.1, $\textsf{G}^{(3)}(p)$ have apparent small world characteristics, e.g., $C_{\text{avg}}^{(3)}(0.07) = 0.5488$ and $L_{\text{avg}}^{(3)}(0.07)=3.5748$, over the 100 vertices. In addition, we find that the average clustering coefficient pattern is similar to tensor entropy. On the contrary, the average path length has a different trend. It decreases faster at small $p$ and more slowly than the average clustering coefficient at large $p$, see Figure \ref{fig:3.5}C.

\subsection{Primary school contact}
The primary school contact dataset contains the temporal network of face-to-face contacts amongst the children and teachers (242 people in total) at a primary school, in which an active contact can include more than two people \cite{10.1186/s12879-014-0695-9,10.1371/journal.pone.0023176}. In this study, we consider the cases of two-person contacts (i.e., a normal graph) and three-person contacts (i.e., a 3-uniform hypergraph) per hour, and explore the relations of tensor entropy with contact frequencies and number of people involved over one school day. The results are shown in Figure \ref{fig:5}, in which the two entropies have a similar and reasonable pattern. Both two-person and three-person contacts are more active at the second and eighth hours, and are less active at the fifth hour. Like the von Neumann entropy ($k=2$), the tensor entropy ($k=3$) is expected to grow with increased number of connected vertices, see Figure \ref{fig:5}B and \ref{fig:5}C, which implies that more children and teachers involved will yield larger tensor entropies. On the other hand, the entropy also heavily relies on the complexity and regularity of the uniform hypergraphs as demonstrated before. For instance, the number of people involved at the seventh hour is greater than that at the fifth hour for $k=3$, but the tensor entropies are opposite because more contacts are made at the fifth hour, increasing the complexity or regularity in the uniform hypergraph, see Figure \ref{fig:5}A. 

\subsection{Mouse neuron endomicroscopy}
The goal of the experiment is to observe mouse neuron activation patterns using fluorescence across space and time before and after food treatment in the mouse hypothalamus. Large changes in fluorescence are inferred to be active neurons that are ``firing.'' The mouse endomicroscopy dataset is an imaging video created under the 10-minute periods of feeding, fasting and re-feeding. The imaging region contains in total 20 neuron cells, and the levels of ``firing'' are also recorded for each neuron. In this study, we build $k$-uniform hypergraphs for each 10-minute interval based on the correlations/multi-correlations of the neuron ``firing'' level for $k=2,3$. The multi-correlation among three variables is defined by
\begin{equation}\label{eq:99}
r^2 = c_1^2 + c_2^2 + c_3^2 - 2c_1c_2c_3, 
\end{equation} 
where, $c_1$, $c_2$ and $c_3$ are the correlations between the three variables \cite{1401.4827v5}. It turns out that the multi-correlation is a generalization of Pearson correlation which can measure the strength of multivariate correlation. 

\begin{figure}[ht]
\centering
\includegraphics[width=\linewidth]{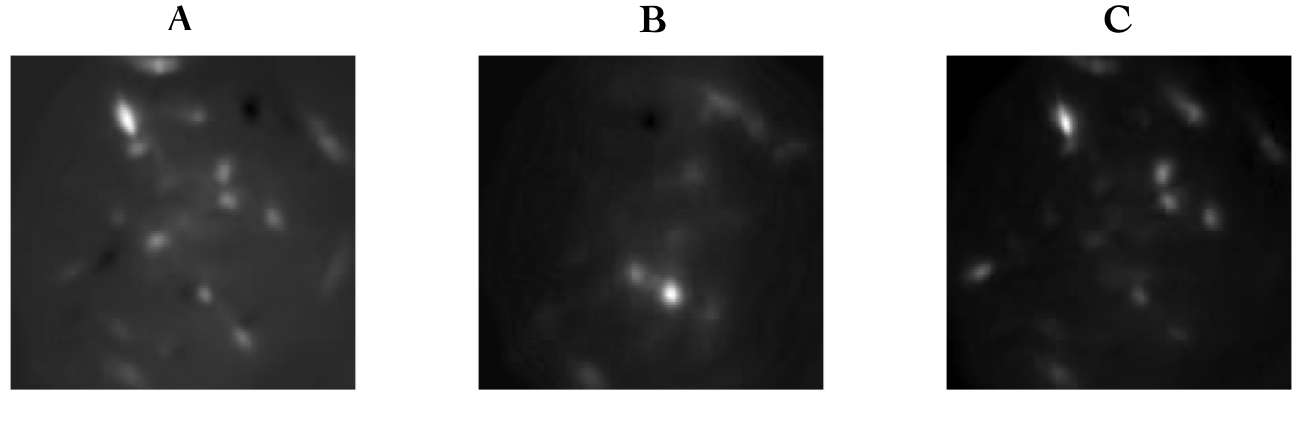}  

\begin{tikzpicture}[scale=0.78]
\begin{axis}[ybar, enlargelimits=0.25, ymax =2.8, ymin=0.5,
    ylabel={\textbf{Tensor Entropy}}, 
    symbolic x coords={A,B,C},
    xtick=data, ylabel near ticks, xticklabels={\textbf{A},\textbf{B},\textbf{C}}, title=\textbf{D}, xlabel=\textbf{Phases}, legend pos = outer north east,,xtick pos=left,ytick pos=left
    ]
\addplot+[blue, mark options={draw=blue,fill=blue}, opacity=.7] coordinates {(A,0.6931) (B,0) (C,1.0114)};
\addlegendentry{$k=2$ w.t. 0.88}
\addplot+[black, mark options={draw=black,fill=black}, opacity=.7] coordinates {(A,0) (B,0) (C,0)};
\addlegendentry{$k=2$ w.t. 0.93}
\addplot+[red, mark options={draw=red,fill=red}, opacity=.7] coordinates {(A,2.6617) (B,2.6193) (C,2.7538)};
\addlegendentry{$k=3$ w.t. 0.88}
\addplot+[orange, mark options={draw=orange,fill=orange}, opacity=.7] coordinates {(A,2.3277) (B,1.3358) (C,2.5503)};
\addlegendentry{$k=3$ w.t. 0.93}

\end{axis}
\end{tikzpicture}
\caption{\textbf{Mouse neuron endomicroscopy features.} (A), (B) and (C) First eigenfaces of the three phases - fed, fast and re-fed. (D) Tensor entropies of the $k$-uniform hypergraphs constructed from the corresponding three phases with $k=2,3$ (here w.t. stands for ``with threshold'').}
\label{fig:6}
\end{figure}
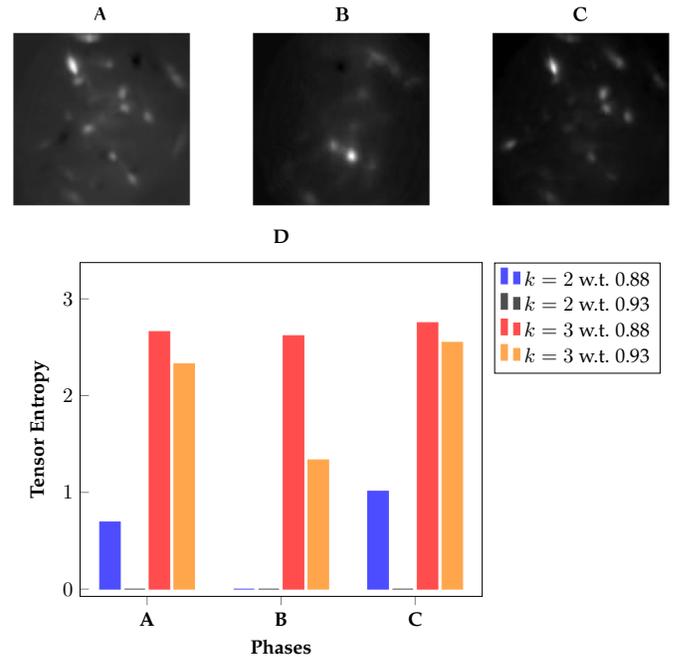

The results are shown in Figure \ref{fig:6}, in which (A), (B) and (C) are the first eigenfaces of the corresponding three phases showing the dominant features in these phases. For computing the entropy, we choose the cutoff threshold to be 0.93 in the construction of edge/hyperedge. In Figure \ref{fig:6}D, the von Neumann entropy ($k=2$) stays constant because the threshold is too high to generate edges in the graph model. However, the tensor entropy ($k=3$) is able to capture changes in neuronal activity, which is lower during the fast phase and higher during the fed/re-fed phase. If we lower the threshold, a similar pattern is observed for $k=2$. To maintain the model accuracy, we want to keep the threshold as high as possible. This supports use of tensor entropy over von Neumann entropy. As validation for using tensor entropy in biological data, we find that  mouse neuron activation patterns can be more accurately captured through 3-uniform hypergraphs. Under the threshold 0.93, the two hypergraphs for the fed and re-fed phases contain a number of common hyperedges. These hyperedges are mainly composed of vertices with high degrees, representing scenarios where more than two neurons synchronize, or ``co-fire'', in the mouse hypothalamus. This suggests that these neurons are involved in mouse appetite regulation, which is not captured using the graph model.

\subsection{Cellular reprogramming}
Cellular reprogramming is a process that introduces proteins called transcription factors as a control mechanism for transforming one cell type into another. The unbiased genome-wide technology of chromosome conformation capture (Hi-C) has been used to capture the dynamics of reprogramming \cite{Rajapakse711,RIED20171,liu_2018}. However, the pairwise contacts from Hi-C data fail to include the multiway interactions of chromatin. Furthermore, the notion of transcription factories supports the existence of simultaneous interactions involving mutiple genomic loci \cite{cook_2018}, implying that the human genome configuration can be represented by a hypergraph. Therefore, in this example, we use 3-uniform hypergraphs to partially recover the 3D configuration of the genome based on the multi-correlation (\ref{eq:99}) from Hi-C matrices. We believe that such reconstruction can provide 
more information about genome structure and patterns, compared to the pairwise Hi-C contacts. We use a cellular reprogramming dataset, containing normalized Hi-C data from fibroblast proliferation and MyoD-mediated fibroblast reprogramming (MyoD is the transcription factor used for control) for Chromosome 14 at 1MB resolution with a total of 89 genomic loci. Our goal is to quantitatively detect a bifurcation in the fibroblast proliferation and reprogramming data, and accurately identify the critical transition point between cell identities during reprogramming. The results are shown in Figure \ref{fig:11}. We can clearly observe a bifurcation between the two trajectories using the tensor entropy of the 3-uniform hypergraphs recovered from the Hi-C measurements. Crucially, the critical transition point marked in Figure \ref{fig:11}A is consistent with the ground-truth statistic provided in \cite{liu_2018}. In contrast, the von Neumann entropy cannot provide adequate information about the bifurcation and critical transition point, if one analyzes the Hi-C measurements as adjacency matrices. The two trajectories are separate from the beginning, see Figure \ref{fig:11}B.

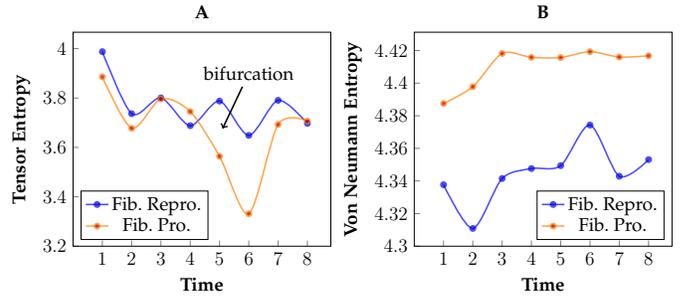
\begin{figure}[ht]
\centering
\begin{tikzpicture}[scale=0.5][font=\Large]
\begin{axis}[ylabel near ticks,
xlabel=\textbf{Time}, xtick={1,2,3,4,5,6,7,8}, ylabel=\textbf{Tensor Entropy}, xmin=0, ymin=3.2,title=\Large\textbf{A},legend pos = south west,xtick pos=left,ytick pos=left]
\addplot+[smooth,very thick,mark=*,opacity=.7] plot coordinates
{ (1,3.9876) (2,3.7366) (3,3.8004) (4, 3.6880) (5,3.7877) (6,3.6486) (7,3.7906) (8,3.6973) };
\addlegendentry{Fib. Repro.}
\addplot+[smooth,very thick,mark=*,opacity=.7,orange] plot coordinates
{ (1,3.8857) (2,3.6776) (3,3.7963) (4, 3.7451) (5,3.5641) (6,3.3314) (7,3.6927) (8,3.7067) };
\addlegendentry{Fib. Pro.}
\node (dest) at (5,3.65) {};
\node (start) at (6, 3.9) {bifurcation};
\draw[->, very thick](start)--(dest);
\end{axis}
\end{tikzpicture}
\begin{tikzpicture}[scale=0.5][font=\Large]
\begin{axis}[ylabel near ticks,
xlabel=\textbf{Time}, xtick={1,2,3,4,5,6,7,8}, ylabel=\textbf{Von Neumann Entropy}, xmin=0, ymin=4.3,title=\Large\textbf{B},legend pos = south east,xtick pos=left,ytick pos=left]
\addplot+[smooth,very thick,mark=*,opacity=.7] plot coordinates
{ (1,4.3377) (2,4.3109) (3,4.3415) (4,4.3476) (5,4.3494) (6, 4.3743) (7,4.3429) (8,4.3531) };
\addlegendentry{Fib. Repro.}
\addplot+[smooth,very thick,mark=*,opacity=.7,orange] plot coordinates
{ (1,4.3875) (2,4.3978) (3,4.4182) (4,4.4158) (5,4.4157) (6,4.4193) (7,4.4160) (8,4.4168) };
\addlegendentry{Fib. Pro.}
\end{axis}
\end{tikzpicture}
\caption{\textbf{Cellular reprogramming features.} (A) Tensor entropies of the uniform hypergraphs recovered from Hi-C measurements with multi-correlation cutoff threshold 0.95. (B) Von Neuman entropies of the binarized Hi-C matrices with weight cutoff threshold 0.95.}
\label{fig:11}
\end{figure}

\subsection{Algorithm run time comparison}\label{sec:train}
In this example, the $k$-uniform hypergraphs are constructed with $n$ vertices by forming a strip structure in which every pair of connected hyperedges only contains one common vertex. We compare the computational efficiency of the SVD-based Algorithm \ref{alg:1} and the TTD-based Algorithm \ref{alg:2} in computing the tensor entropy. The results are shown in Figure \ref{fig:9}. For the TTD-based entropy computations, we assume that all the adjacency and degree tensors of the uniform hypergraphs are already provided in the TT-format. Evidently, Algorithm \ref{alg:2} is more time efficient than Algorithm \ref{alg:1} for 4-uniform and 5-uniform hypergraphs with the strip structure as $n$ becomes larger, see Figure \ref{fig:9}. Particularly, when $k=5$, the TTD-based algorithm exhibits a huge time advantage as predicted in the computation complexity. The time curve from the SVD-based Algorithm \ref{alg:1} increases exponentially, while it grows at a much slower rate if using Algorithm \ref{alg:2}. In the meantime, we compute the relative errors between the tensor entropies computed from the two algorithms, all of which are within $10^{-14}$.   

\begin{figure}[ht]
\begin{tikzpicture}[scale=1]
\begin{axis}[ymode=log, ytick={1e-5, 1e-4, 1e-3, 1e-2, 1e-1, 1, 10, 100}, xlabel={\textbf{Num. of Vertices $\boldsymbol{n}$}}, ylabel={\textbf{Computational Time (sec)}}, ylabel near ticks, legend pos=north west, xlabel near ticks,xtick pos=left,ytick pos=left,xtick align=center]
\addplot+[smooth, very thick,mark=*,blue, mark options={draw=blue,fill=blue}, opacity=.7] coordinates {(4,0.0019) (7,0.0052) (10,0.0035) (13,0.0059) (16,0.0047) (19,0.0031) (22,0.0086) (25,0.0052) (28,0.0067) (31,0.0056) (34,0.0081) (37,0.0080) (40,0.0106) (43,0.0096) (46,0.0126) (49,0.0160) (52,0.0192) (55,0.02) (58,0.0216) (61,0.0261)};
\addlegendentry{$k=4$ (TTD)}
\addplot+[smooth, very thick,mark=*, dashed, blue, mark options={draw=blue,fill=blue}, opacity=.7] coordinates {(4,0.0001527) (7,0.0027) (10,0.0025) (13,0.0036) (16,0.0020) (19,0.0068) (22,0.0133) (25,0.0103) (28,0.0174) (31,0.0217) (34,0.0255) (37,0.0530) (40,0.0636) (43,0.0835) (46,0.1119) (49,0.1511) (52,0.2162) (55,0.2530) (58,0.3306) (61,0.4059)};
\addlegendentry{$k=4$ (SVD)}
\addplot+[smooth, very thick,mark=*,orange, mark options={draw=orange,fill=orange}, opacity=.7] coordinates {(5,0.003) (9,0.002) (13,0.0026) (17,0.0032) (21,0.0036) (25,0.0046) (29,0.0057) (33,0.0047) (37,0.0105) (41,0.0146) (45,0.0135) (49,0.0197) (53,0.0336) (57,0.0499) (61,0.0502)};
\addlegendentry{$k=5$ (TTD)}
\addplot+[smooth, dashed, very thick,mark=*,orange, mark options={draw=orange,fill=orange}, opacity=.7] coordinates {(5,0.000652) (9,0.0026) (13,0.0132) (17,0.0319) (21,0.0881) (25,0.2383) (29,0.5461) (33,1.1114) (37,1.9470) (41,3.4269) (45,5.8612) (49,11.4317) (53,19.5933) (57,45.0417) (61,70.8018)};
\addlegendentry{$k=5$ (SVD)}
\end{axis}
\end{tikzpicture}
\caption{\textbf{Computational time comparisons between the SVD-based and TTD-based algorithms.} For the TTD-based entropy computation, we reported the times of left- and right-orthonormalization and economy-size matrix SVD. For the SVD-based entropy computation, we only reported the time of economy-size matrix SVD. For the purpose of accuracy, we ran each algorithm 10 times and took the average of the computational times.}
\label{fig:9}
\end{figure}
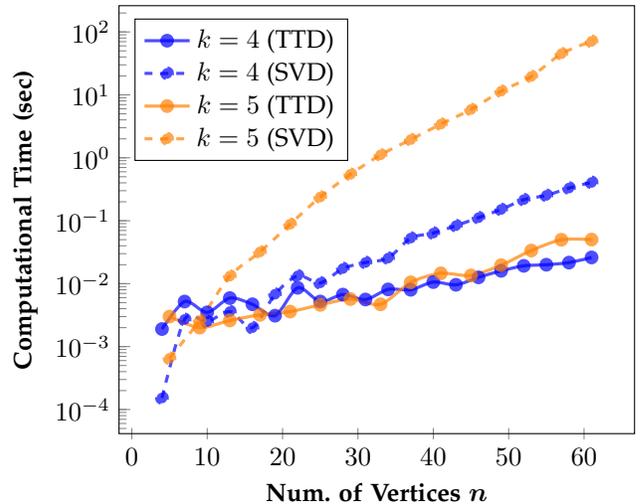

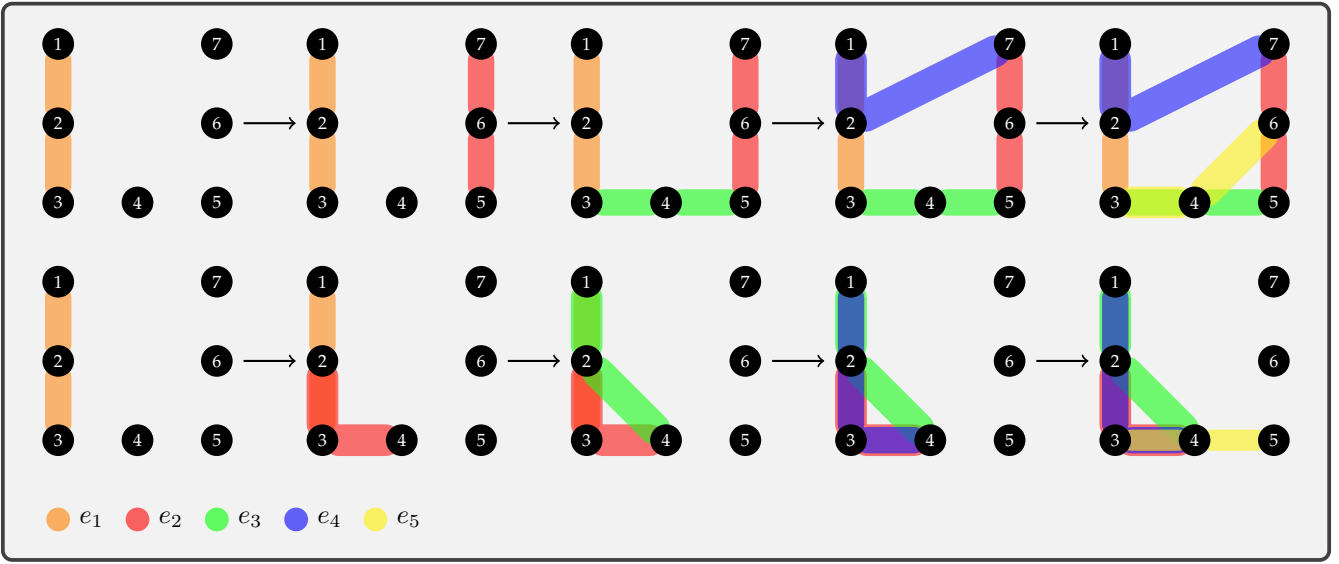
\begin{figure*}[htb!]
\centering
\tcbox[colback=mygray]{
\begin{tikzpicture}
\node[vertex,text=white,scale=0.7] (v1) {1};
\node[vertex,below of=v1,text=white,scale=0.7] (v2) {2};
\node[vertex,below of=v2,text=white,scale=0.7] (v3) {3};
\node[vertex,right of=v3,text=white,scale=0.7] (v4) {4};
\node[vertex,right of=v4,text=white,scale=0.7] (v5) {5};
\node[vertex,above of=v5,text=white,scale=0.7] (v6) {6};
\node[vertex,above of=v6,text=white,scale=0.7] (v7) {7};

\node[vertex,right of=v7, text=white,scale=0.7,xshift=5mm] (v8) {1};
\node[vertex,right of=v6,text=white,scale=0.7,xshift=5mm] (v9) {2};
\node[vertex,right of=v5,text=white,scale=0.7,xshift=5mm] (v10) {3};
\node[vertex,right of=v10,text=white,scale=0.7] (v11) {4};
\node[vertex,right of=v11,text=white,scale=0.7] (v12) {5};
\node[vertex,above of=v12,text=white,scale=0.7] (v13) {6};
\node[vertex,above of=v13,text=white,scale=0.7] (v14) {7};

\node[vertex,right of=v14, text=white,scale=0.7,xshift=5mm] (v15) {1};
\node[vertex,right of=v13,text=white,scale=0.7,xshift=5mm] (v16) {2};
\node[vertex,right of=v12,text=white,scale=0.7,xshift=5mm] (v17) {3};
\node[vertex,right of=v17,text=white,scale=0.7] (v18) {4};
\node[vertex,right of=v18,text=white,scale=0.7] (v19) {5};
\node[vertex,above of=v19,text=white,scale=0.7] (v20) {6};
\node[vertex,above of=v20,text=white,scale=0.7] (v21) {7};

\node[vertex,right of=v21, text=white,scale=0.7,xshift=5mm] (v22) {1};
\node[vertex,right of=v20,text=white,scale=0.7,xshift=5mm] (v23) {2};
\node[vertex,right of=v19,text=white,scale=0.7,xshift=5mm] (v24) {3};
\node[vertex,right of=v24,text=white,scale=0.7] (v25) {4};
\node[vertex,right of=v25,text=white,scale=0.7] (v26) {5};
\node[vertex,above of=v26,text=white,scale=0.7] (v27) {6};
\node[vertex,above of=v27,text=white,scale=0.7] (v28) {7};

\node[vertex,right of=v28, text=white,scale=0.7,xshift=5mm] (v29) {1};
\node[vertex,right of=v27,text=white,scale=0.7,xshift=5mm] (v30) {2};
\node[vertex,right of=v26,text=white,scale=0.7,xshift=5mm] (v31) {3};
\node[vertex,right of=v31,text=white,scale=0.7] (v32) {4};
\node[vertex,right of=v32,text=white,scale=0.7] (v33) {5};
\node[vertex,above of=v33,text=white,scale=0.7] (v34) {6};
\node[vertex,above of=v34,text=white,scale=0.7] (v35) {7};

\node[vertex,below of=v3, text=white,scale=0.7] (w1) {1};
\node[vertex,below of=w1,text=white,scale=0.7] (w2) {2};
\node[vertex,below of=w2,text=white,scale=0.7] (w3) {3};
\node[vertex,right of=w3,text=white,scale=0.7] (w4) {4};
\node[vertex,right of=w4,text=white,scale=0.7] (w5) {5};
\node[vertex,above of=w5,text=white,scale=0.7] (w6) {6};
\node[vertex,above of=w6,text=white,scale=0.7] (w7) {7};

\node[vertex,right of=w7, text=white,scale=0.7,xshift=5mm] (w8) {1};
\node[vertex,right of=w6,text=white,scale=0.7,xshift=5mm] (w9) {2};
\node[vertex,right of=w5,text=white,scale=0.7,xshift=5mm] (w10) {3};
\node[vertex,right of=w10,text=white,scale=0.7] (w11) {4};
\node[vertex,right of=w11,text=white,scale=0.7] (w12) {5};
\node[vertex,above of=w12,text=white,scale=0.7] (w13) {6};
\node[vertex,above of=w13,text=white,scale=0.7] (w14) {7};

\node[vertex,right of=w14, text=white,scale=0.7,xshift=5mm] (w15) {1};
\node[vertex,right of=w13,text=white,scale=0.7,xshift=5mm] (w16) {2};
\node[vertex,right of=w12,text=white,scale=0.7,xshift=5mm] (w17) {3};
\node[vertex,right of=w17,text=white,scale=0.7] (w18) {4};
\node[vertex,right of=w18,text=white,scale=0.7] (w19) {5};
\node[vertex,above of=w19,text=white,scale=0.7] (w20) {6};
\node[vertex,above of=w20,text=white,scale=0.7] (w21) {7};

\node[vertex,right of=w21, text=white,scale=0.7,xshift=5mm] (w22) {1};
\node[vertex,right of=w20,text=white,scale=0.7,xshift=5mm] (w23) {2};
\node[vertex,right of=w19,text=white,scale=0.7,xshift=5mm] (w24) {3};
\node[vertex,right of=w24,text=white,scale=0.7] (w25) {4};
\node[vertex,right of=w25,text=white,scale=0.7] (w26) {5};
\node[vertex,above of=w26,text=white,scale=0.7] (w27) {6};
\node[vertex,above of=w27,text=white,scale=0.7] (w28) {7};

\node[vertex,right of=w28, text=white,scale=0.7,xshift=5mm] (w29) {1};
\node[vertex,right of=w27,text=white,scale=0.7,xshift=5mm] (w30) {2};
\node[vertex,right of=w26,text=white,scale=0.7,xshift=5mm] (w31) {3};
\node[vertex,right of=w31,text=white,scale=0.7] (w32) {4};
\node[vertex,right of=w32,text=white,scale=0.7] (w33) {5};
\node[vertex,above of=w33,text=white,scale=0.7] (w34) {6};
\node[vertex,above of=w34,text=white,scale=0.7] (w35) {7};

\path [->,shorten >=4pt,shorten <=4pt, thick](v6) edge node[left] {} (v9);
\path [->,shorten >=4pt,shorten <=4pt,thick](v13) edge node[left] {} (v16);
\path [->,shorten >=4pt,shorten <=4pt,thick](v20) edge node[left] {} (v23);
\path [->,shorten >=4pt,shorten <=4pt,thick](v27) edge node[left] {} (v30);

\path [->,shorten >=4pt,shorten <=4pt,thick](w6) edge node[left] {} (w9);
\path [->,shorten >=4pt,shorten <=4pt,thick](w13) edge node[left] {} (w16);
\path [->,shorten >=4pt,shorten <=4pt,thick](w20) edge node[left] {} (w23);
\path [->,shorten >=4pt,shorten <=4pt,thick](w27) edge node[left] {} (w30);

\begin{pgfonlayer}{background}
\begin{scope}[transparency group,opacity=.9]
\draw[edge,color=orange] (v1) -- (v2) -- (v3);

\draw[edge,color=orange] (v8) -- (v9) -- (v10);
\draw[edge,color=red] (v12) -- (v13) -- (v14);

\draw[edge,color=orange] (v15) -- (v16) -- (v17);
\draw[edge,color=red] (v19) -- (v20) -- (v21);
\draw[edge,color=green] (v17) -- (v18) -- (v19);

\draw[edge,color=orange] (v22) -- (v23) -- (v24);
\draw[edge,color=red] (v26) -- (v27) -- (v28);
\draw[edge,color=green] (v24) -- (v25) -- (v26);
\draw[edge,color=blue,line width=12pt] (v22) -- (v23) -- (v28);

\draw[edge,color=orange] (v29) -- (v30) -- (v31);
\draw[edge,color=red] (v33) -- (v34) -- (v35);
\draw[edge,color=green] (v31) -- (v32) -- (v33);
\draw[edge,color=blue,line width=12pt] (v29) -- (v30) -- (v35);
\draw[edge,color=yellow,line width=12pt] (v31) -- (v32) -- (v34);

\draw[edge,color=orange] (w1) -- (w2) -- (w3);

\draw[edge,color=orange] (w8) -- (w9) -- (w10);
\draw[edge,color=red,line width=12pt] (w9) -- (w10) -- (w11);

\draw[edge,color=orange] (w15) -- (w16) -- (w17);
\draw[edge,color=red,line width=12pt] (w16) -- (w17) -- (w18);
\draw[edge,color=green,line width=12pt] (w15) -- (w16) -- (w18);

\draw[edge,color=orange] (w22) -- (w23) -- (w24);
\draw[edge,color=red,line width=12pt] (w23) -- (w24) -- (w25);
\draw[edge,color=green,line width=12pt] (w22) -- (w23) -- (w25);
\draw[edge,color=blue,line width=10pt] (w22) -- (w24) -- (w25);

\draw[edge,color=orange] (w29) -- (w30) -- (w31);
\draw[edge,color=red,line width=12pt] (w30) -- (w31) -- (w32);
\draw[edge,color=green,line width=12pt] (w29) -- (w30) -- (w32);
\draw[edge,color=blue,line width=10pt] (w29) -- (w31) -- (w32);
\draw[edge,color=yellow,line width=8pt] (w31) -- (w32) -- (w33);

\end{scope}
\end{pgfonlayer}
\node[elabel,color=orange,fill,opacity=.6,fill opacity=.6,below of=w3, label=right:\(e_1\)]  (e1) {};
\node[elabel,right of=e1,color=red,fill,opacity=.6,fill opacity=.6,label=right:\(e_2\)]  (e2) {};
\node[elabel,right of=e2,color=green,fill,opacity=.6,fill opacity=.6,label=right:\(e_3\)]  (e3) {};
\node[elabel,right of=e3,color=blue,fill,opacity=.6,fill opacity=.6,label=right:\(e_4\)]  (e4) {};
\node[elabel,right of=e4,color=yellow,fill,opacity=.6,fill opacity=.6,label=right:\(e_5\)]  (e5) {};

\end{tikzpicture}}
\caption{\textbf{Eigenvalue entropy maximization/minimization.} The top row describes the first five stages of the eigenvalue entropy maximization evolution with a growing number of hyperedges in the order of $e_1=\{1,2,3\}$, $e_2=\{5,6,7\}$, $e_3=\{3,4,5\}$, $e_4=\{1,2,7\}$ and $e_ 5=\{3,4,6\}$. The eigenvalue entropy $\textsc{S}_{\max}^{(j)}$ = 4.4910, 5.6342, 5.8608, 5.9490 and 6.0091 for $j=1,2,3,4,5$. The bottom row reports the first five stages of the eigenvalue entropy minimization process with a growing number of hyperedges in the order of $e_1=\{1,2,3\}$, $e_2=\{2,3,4\}$, $e_3=\{1,2,4\}$, $e_4=\{1,3,4\}$ and $e_ 5=\{3,4,5\}$. The eigenvalue entropy $\textsc{S}_{\min}^{(j)}$ = 4.4910, 5.3604, 5.4434, 5.4715 and 5.6334 for $j=1,2,3,4,5$. All the tensor eigenvalues of the Laplacian tensors in this experiment are computed from the MATLAB Toolbox TenEig \cite{teneig,doi:10.1137/15M1010725}.}
\label{fig:2}
\end{figure*}

\section{Discussion}\label{sec:dis}
The first five numerical examples reported here highlight that the $k$-mode singular values computed from the HOSVD of the Laplacian tensors can provide nice predictions of structural properties for uniform hypergraphs. This method can also be used for anomaly detection in the context of dynamics as we demonstrated in the mouse neuron endomicroscopy and cellular reprogramming datasets. However, more theoretical and numerical investigations are required to assess the real advantages of hypergraphs versus normal graphs, and is an important avenue of future research. Moreover, as we pointed out in section \ref{sec:2.3}, many simple structure tensors can be directly created in the TT-format without requiring construction of the full representations. For example, Oseledets et al. \cite{tttoolbox} built the Laplacian operator in the TT-format for the discretized heat equations. We believe that similar results can happen to the adjacency, degree and Laplacian tensors for uniform hypergraphs. 

Instead of looking at the $k$-mode singular values, we can consider the tensor eigenvalues in defining the tensor entropy. We will refer to it as the eigenvalue entropy later. See Appendix B for a short introduction to tensor eigenvalues for supersymmetric tensors. Based on the tensor eigenvalue formulations, we can establish the eigenvalue entropy measure for uniform hypergraphs.
\begin{definition}
Let \textsf{G} be a $k$-uniform hypergraph with $n$ vertices. The eigenvalue entropy of \textsf{G} is defined by
\begin{equation}\label{eq:13}
\textsc{S} = -\sum_{j=1}^{d} |\hat{\lambda}_j|\ln{|\hat{\lambda}_j|},
\end{equation}
where, $|\hat{\lambda}_j|$ are the normalized modulus of the eigenvalues of the Laplacian tensor \textsf{L} such that $\sum_{j=1}^d|\hat{\lambda}_j| = 1$, and $d=n(k-1)^{n-1}$ is the total number of the eigenvalues.
\end{definition}
The convention $0\ln{0}=0$ is used if $|\hat{\lambda}_j|=0$. We can use other tensor eigenvalue notions including H-eigenvalue, E-eigenvalue and Z-eigenvalue with a corresponding $d$ to fit into the formula. For curiosity, we repeat the hyperedge growth model using the eigenvalue entropy, see Figure \ref{fig:2}. The entropy minimization evolution trajectory is the same for the first five stages  in which \textit{cliques} are formed. The maximization evolution trajectory becomes different from the fourth stage after the hypergraph is connected, in which short path lengths and high level of nontrivial symmetricity are no longer the factors that maximize the entropy. In addition, computing the eigenvalues of a tensor is an NP hard problem \cite{Hillar:2013:MTP:2555516.2512329}.  Hence, the eigenvalue entropy may not be used to predict the structural properties for uniform hypergraphs, but it might contain other unknown features that are required to explore in the future.


Furthermore, the notion of robustness for uniform hypergraphs is an important topic. In graph theory, one of the popular measures is called effective resistance \cite{ELLENS20112491}. The authors in \cite{resistence} show that the effective graph resistance can be written in terms of the reciprocals of graph Laplacian eigenvalues, and robust networks have small effective graph resistance. Hence, we attempt to establish similar relationship using the $k$-mode singular values from the Laplacian tensors to describe the robustness of uniform hypergraphs. 

\begin{definition}
Let \textsf{G} be a connected $k$-uniform hypergraph with $n$ vertices. The effective resistance of \textsf{G} is defined by
\begin{equation}\label{eq:20}
    \textsc{R} = n\sum_{j=1}^n \frac{1}{\gamma_j},
\end{equation}
where, $\gamma_j$ are the $k$-mode singular values of \textsf{L}.
\end{definition}

If a uniform hypergraph is non-connected, then the effective resistance $\textsc{R}=\infty$.  Based on (\ref{eq:20}), we compute the effective resistance of the uniform hypergraphs $\textsf{G}_{\max}^{(j)}$ in the hyperedge growth model example, denoted by $\textsc{R}_{\max}^{(j)}$.  Similar to the effective graph resistance, $\textsc{R}_{\max}^{(j)}$ strictly decreases when hyperedges are added and achieves the minimum at the final step when the hypergraph is complete, see Table \ref{tab:1}. We can also expect that the smaller the effective resistance is, the more robust the uniform hypergraph. We believe that the effective resistance (\ref{eq:20}) is a good measure for uniform hypergraph robustness, but more theoretical and numerical support is needed to verify this hypothesis. 

\begin{table}[ht]
\caption{\textbf{The effective resistance and tensor entropy.}}
\centering
\begin{tabular}{|l|l|l|l|l|l|}
\hline
Step & $j=3$ & $j=5$ & $j=15$ & $j=25$ & $j=35$ \\ \hline
$\textsc{R}_{\max}^{(j)}$ &  34.8384   &  20.9432   &  7.3863   &  4.4819   & 3.2166    \\ \hline
$\textsc{S}_{\max}^{(j)}$ &  1.9037   &  1.9359   & 1.9430    & 1.9448    &  1.9456   \\ \hline
\end{tabular}
\label{tab:1}
\end{table}

\section{Conclusion}\label{sec:5}
In this paper, we proposed a new notion of entropy for uniform hypergraphs based on the tensor higher-order singular value decomposition. The $k$-mode singular values of Laplacian tensors provide nice interpretations regarding the structural properties of uniform hypergraphs. The tensor entropy heavily depends on the vertex degrees, path lengths, clustering coefficients and nontrivial symmetricity. We investigated the lower and upper bounds of the entropy and provided the entropy formula for complete uniform hypergraphs. A TTD-based computational framework is proposed for computing the tensor entropy efficiently. We also applied this spectral measure to real biological networks for anomaly detection, and achieve better performances compared to the von Neumann graph entropy. As discussed in section \ref{sec:dis}, the detailed relations between tensor eigenvalues and entropy, and the theoretical investigations of hypergraph robustness require further exploration. Controllability and influenceability of uniform hypergraphs are also important for future research.


%

%

\section*{Acknowledgment}
We thank Dr. Frederick Leve at the Air Force Office of Scientific Research (AFOSR) for support and encouragement. This work is supported in part by AFOSR Award No: FA9550-18-1-0028, the Smale Institute, the DARPA/MTO Lifelong Learning Machines program, and the  DARPA/ITO Guaranteeing AI Robustness Against Deception program.

\ifCLASSOPTIONcaptionsoff
  \newpage
\fi



%

%

%
\begin{IEEEbiography}[{\includegraphics[width=1in,height=1.25in,clip,keepaspectratio]{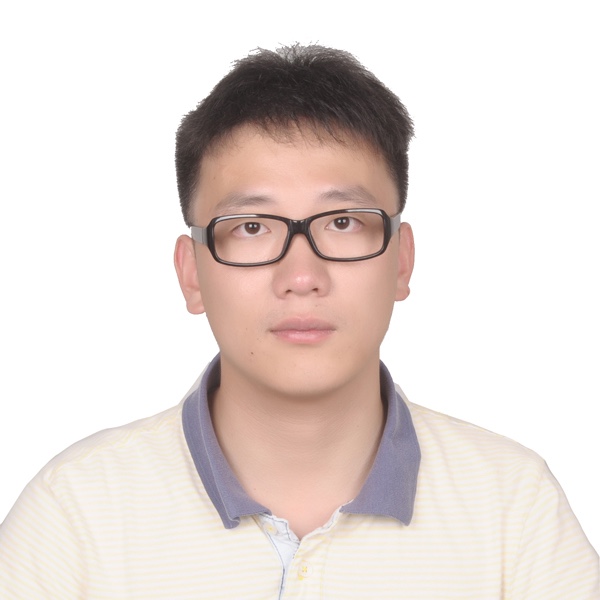}}]{Can Chen} is currently a Ph.D. candidate in Applied \& Interdisciplinary Mathematics and a Master's student in Electrical \& Computer Engineering at the University of Michigan, Ann Arbor. He received B.S. degree in Mathematics at the University of California, Irvine in 2016. His research is focused on data-guided control of multiway dynamical systems. 
\end{IEEEbiography}

\begin{IEEEbiography}[{\includegraphics[width=1in,height=1.25in,clip,keepaspectratio]{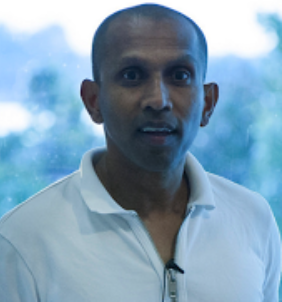}}]{Indika Rajapakse} is currently an Associate Professor of Computational Medicine \& Bioinformatics, in the Medical School, and an Associate Professor of Mathematics at the University of Michigan, Ann Arbor. He is also a member of the Smale Institute. His research is at the interface of biology, engineering and mathematics. His areas include dynamical systems, networks, mathematics of data and cellular reprogramming.

\end{IEEEbiography}


%
%
%


\vfill


\end{document}